\documentclass[sigconf, nonacm]{acmart}

\AtBeginDocument{%
  }

\usepackage{algorithm}
\usepackage{algpseudocode}
\usepackage{amsmath}
\usepackage{multirow}
\setcitestyle{numbers,comma}

\makeatletter
\newcommand{\doi}[1]{\url{https://doi.org/#1}}
\makeatother

\begin{document}

\title{Trustworthy Protein–Ligand Binding Affinity Prediction via Reliability-Aware Multi-Engine Fusion}

\author{Yongchan Hong}
\orcid{0009-0009-8866-1690}
\affiliation{%
  \department{Department of Quantitative \& Computational Biology}
  \institution{University of Southern California}
  \city{Los Angeles}
  \state{California}
  \country{USA}
}
\email{hongyong@usc.edu}

\author{Defu Cao}
\orcid{0000-0003-0240-3818}
\affiliation{%
  \department{Department of Computer Science}
  \institution{University of Southern California}
  \city{Los Angeles}
  \state{California}
  \country{USA}
}
\email{defucao@usc.edu}

\author{Wenjin Liu}
\orcid{0000-0003-1909-8135}
\affiliation{%
  \department{Department of Quantitative \& Computational Biology}
  \institution{University of Southern California}
  \city{Los Angeles}
  \state{California}
  \country{USA}
}
\email{wenjinl@usc.edu}

\author{Thomas Ku}
\orcid{0009-0006-1647-2979}
\affiliation{%
  \department{Department of Quantitative \& Computational Biology}
  \institution{University of Southern California}
  \city{Los Angeles}
  \state{California}
  \country{USA}
}
\email{kut@usc.edu}

\author{Jordy Homing Lam}
\orcid{0000-0002-5496-6228}
\affiliation{%
  \department{Department of Chemistry}
  \institution{University of California, Berkeley}
  \city{Berkeley}
  \state{California}
  \country{USA}
}
\email{jhml@berkeley.edu}

\author{Emily Nguyen}
\orcid{0000-0003-4917-7336}
\affiliation{%
  \department{Department of Computer Science}
  \institution{University of Southern California}
  \city{Los Angeles}
  \state{California}
  \country{USA}
}
\email{emilyn98@usc.edu}

\author{Willie Neiswanger}
\orcid{0000-0002-9619-5572}
\affiliation{%
  \department{Department of Computer Science}
  \institution{University of Southern California}
  \city{Los Angeles}
  \state{California}
  \country{USA}
}
\email{neiswang@usc.edu}

\author{Vsevolod Katritch}
\orcid{0000-0003-3883-4505}
\authornote{Corresponding authors. This work has been accepted for publication in the Proceedings of the 32nd ACM SIGKDD Conference on Knowledge Discovery and Data Mining (KDD 2026), August 9--13, 2026, Jeju, South Korea.}
\affiliation{%
  \department{Department of Quantitative \& Computational Biology}
  \institution{University of Southern California}
  \city{Los Angeles}
  \state{California}
  \country{USA}
}
\email{katritch@usc.edu}

\author{Yan Liu}
\orcid{0000-0002-7055-9518}
\authornotemark[1]
\affiliation{%
  \department{Department of Computer Science}
  \institution{University of Southern California}
  \city{Los Angeles}
  \state{California}
  \country{USA}
}
\email{yanliu.cs@usc.edu}

\renewcommand{\shortauthors}{Yongchan Hong et al.}

\begin{abstract}
 Accurate protein–ligand binding affinity prediction is central to computational drug discovery, yet modern docking engines frequently disagree without indicating which prediction to trust. Consensus scoring and ensemble methods improve mean accuracy but treat all predictions identically without interpretable confidence measures or uncertainty decomposition, ignoring the chemical context of each protein–ligand pair.
 To address this limitation, we introduce \textbf{RELIABLE-BA} (\textbf{RELIAB}i\textbf{L}ity-aware \textbf{E}vidential fusion for \textbf{B}inding \textbf{A}ffinity), an evidential framework for multi-engine binding affinity prediction. Our model comprises three steps: (1) modeling each engine as an evidential expert via Normal–Inverse-Gamma distributions, (2) scaling epistemic uncertainty through learned reliability from molecular context while preserving each expert's predictive mean, and (3) fusing experts through closed-form aggregation that captures both individual uncertainty and inter-engine disagreement. Experiments on the PDBBind and BDB2020+ benchmarks demonstrate competitive point prediction with substantially improved uncertainty calibration, and additional validation on the SARS-CoV-2 Mpro dataset and 5HT2A receptor demonstrates applicability to clinically relevant drug targets. Crucially, these uncertainty estimates enable reliable filtering of protein-ligand pairs, reducing prediction error by up to ~25\% when retaining only high-confidence pairs. To our knowledge, RELIABLE-BA is the first multi-engine binding affinity prediction framework to combine evidential fusion with context-dependent reliability, offering a principled path toward trustworthy AI-guided drug discovery. Our code is publicly available at \href{https://github.com/yongchand/RELIABLE-BA}{https://github.com/yongchand/RELIABLE-BA}. 
 
\end{abstract}



\maketitle

\begin{figure}[t]
\centering
\includegraphics[width=0.8\columnwidth]{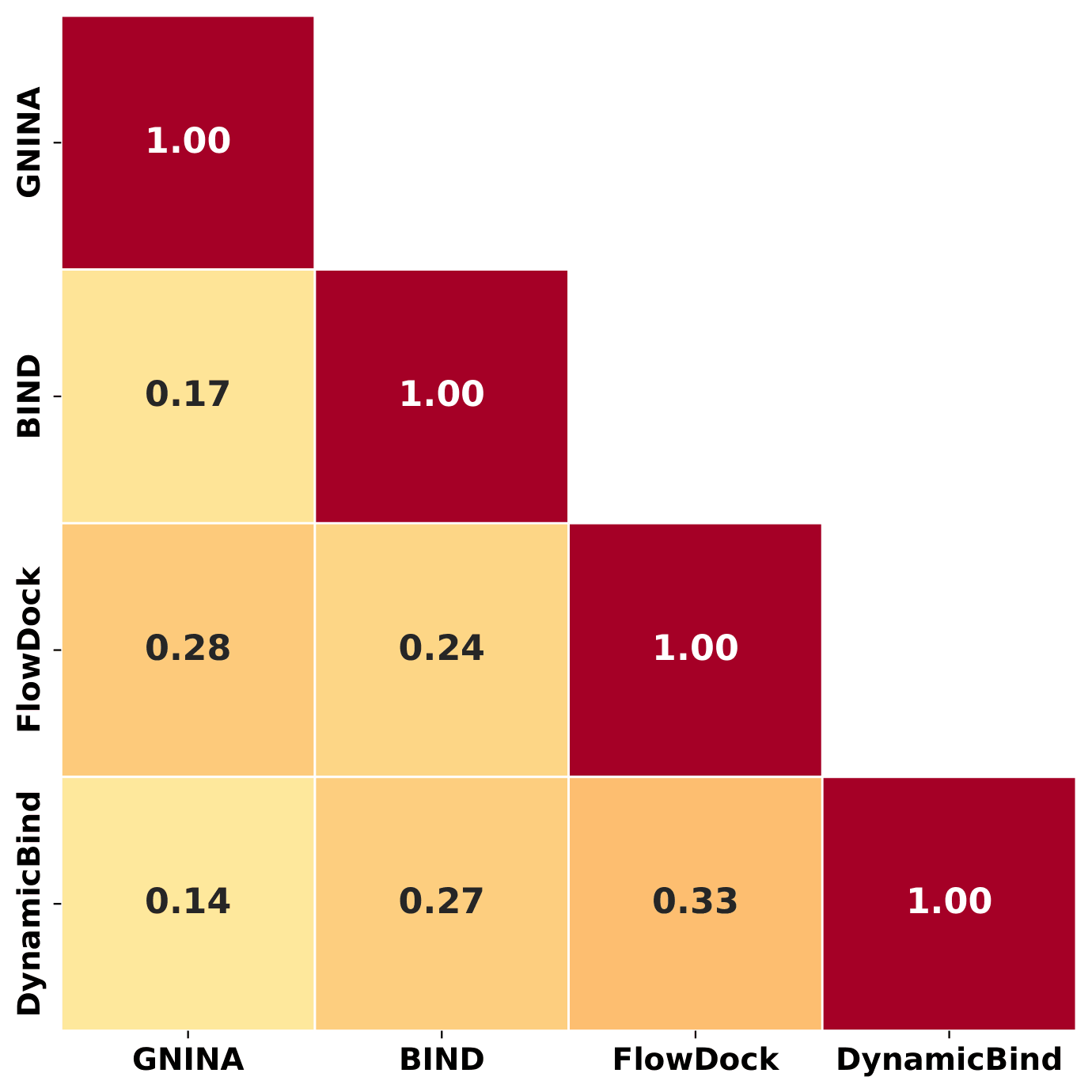}
\caption{Pairwise error correlations between four binding affinity prediction engines on PDBbind. Prediction error is defined as the absolute difference between predicted and experimental binding affinity ($pK_d$/$pK_i$). Spearman correlation coefficients range from 0.14 to 0.33.}
\Description{Error correlation matrix for binding affinity prediction engines.}
\label{fig:error_correlation_matrix}
\end{figure}

\section{Introduction}
Protein–ligand binding affinity prediction is a core component of computational drug discovery, enabling efficient hit discovery and substantially reducing the number of candidate compounds for experimental evaluation~\cite{lam2025navigating}. While physics-based methods such as free energy perturbation (FEP) achieve high accuracy, they require extensive sampling and are computationally prohibitive for large-scale screening~\cite{wang2015accurate}. Machine learning (ML) approaches offer faster inference with competitive accuracy~\cite{cao2026pinfdit}. Boltz-2 approaches FEP-level accuracy with Pearson correlations of 0.62–0.66 on standard benchmarks~\cite{passaro2025boltz2}, demonstrating that deep learning can meaningfully accelerate therapeutic discovery~\cite{Panahandeh2025}. Modern ML-based affinity prediction engines demonstrate distinct strengths. Physics-based methods capture fundamental energetics, structure-based models exploit geometric features, and sequence-based approaches generalize across protein families without requiring 3D structures~\cite{mcnutt2021gnina, bind, lu2024dynamicbind}. Nevertheless, protein–ligand binding is inherently complex, involving subtle interplay of electrostatics, solvation, entropy, and conformational dynamics that no single model fully captures~\cite{buttenschoen2024posebusters}. As a result, even state-of-the-art engines exhibit complementary failure modes: Figure ~\ref{fig:error_correlation_matrix} shows that pairwise error correlations between engines are weak (Spearman $\rho$ = 0.14-0.33), indicating that different engines succeed on different subsets of complexes. This complementarity suggests that fusing existing engines offers a more efficient path to robust predictions than building a single universal model from scratch.

Building on this complementarity strength, practitioners employ multiple docking tools to improve robustness~\cite{Scardino_2022_consensus_survey, Ericksen_2017_MLconsensus}. Existing multi-engine strategies fall into two paradigms: consensus scoring, which combines predictions using heuristic rules such as rank aggregation or score normalization~\cite{Yang_2005_consensus, NhatPhuong_2023_consensus}, and ensemble learning, which aggregates outputs via averaging or meta-learning~\cite{Ballester_2010_RFscore, Uddin_2024_EBA, Rayka_2024_ENSscore}. While these approaches can improve mean accuracy, they are primarily designed to optimize aggregate performance, rather than to reason about engine-specific reliability for individual protein–ligand complexes.

To reason about which predictions to rely on in a given protein–ligand complex, uncertainty quantification (UQ) characterizes predictive confidence rather than producing point estimates alone. This is essential for trustworthy decision-making in drug discovery, where models are frequently applied to novel scaffolds or targets outside the training distribution~\cite{Rayka_2025_UQ_binding}. Calibrated uncertainty estimates enable practitioners to prioritize high-confidence predictions for experimental validation, thereby reducing costly false positives. In active learning settings, uncertainty guides the selection of informative compounds for iterative screening~\cite{Graff_2021_active_learning}. A key distinction is between aleatoric uncertainty, arising from inherent measurement noise that cannot be reduced, and epistemic uncertainty, reflecting limited model knowledge that could be addressed with additional data or improved models~\cite{Kendall_2017_uncertainties,cao2023estimating}. Without this decomposition, practitioners cannot determine whether low confidence stems from inherent target difficulty or correctable gaps in model coverage.

Recent advances in uncertainty quantification provide principled frameworks for single-model settings. Evidential regression using Normal–Inverse-Gamma (NIG) distributions enables closed-form uncertainty decomposition ~\cite{amini2020deep}, and has shown promise in drug discovery applications~\cite{EviDTI_2025_NatComm, Xu_2025_UAMRL}. However, these methods address uncertainty within a single model and do not account for the additional epistemic uncertainty arising from disagreement among engines. Binding affinity engines are not exchangeable as they rely on different assumptions, scoring functions, and training data with varying coverage~\cite{Wang_2016_comprehensive}. Existing single-model UQ methods cannot capture this inter-engine disagreement and determine which engine to trust for a given protein–ligand complex.

To address the above mentioned challenges, we introduce RELIABLE-BA (RELIABiLity-aware Evidential fusion for Binding Affinity). RELIABLE-BA models each docking engine as an evidential expert outputting a NIG distribution that captures both aleatoric and epistemic uncertainty. A reliability network learns context-dependent trustworthiness scores conditioned on protein–ligand features, which modulate each expert's epistemic uncertainty through a principled scaling transformation while preserving predictive means. The reliability-scaled experts are then fused via closed-form Mixture of Normal–Inverse-Gamma (MoNIG) aggregation~\cite{ma2021trustworthy}, yielding calibrated predictions that reflect both individual engine uncertainty and inter-engine disagreement.

We evaluate RELIABLE-BA on PDBbind~\cite{wang2005pdbbind} using a time-split protocol, with additional validation on BDB2020+~\cite{Li_2026_LPPDBBind}, an independent test set with zero training overlap, and a case study on the SARS-CoV-2 Mpro dataset and 5HT2A receptor. Across all settings, RELIABLE-BA achieves competitive point prediction while substantially improving uncertainty calibration. Selective prediction analyses demonstrate consistent error reduction at low coverage, and epistemic uncertainty correlates with inter-engine disagreement while point predictions remain stable. These properties make RELIABLE-BA well-suited for multi-engine virtual screening pipelines where calibrated uncertainty is needed to prioritize high-confidence predictions before committing experimental resources.

In summary, this paper makes the following contributions:
\begin{enumerate}
\item We propose \textbf{RELIABLE-BA}, the first multi-engine binding affinity framework combining evidential uncertainty quantification with context-dependent reliability modeling.
\item We introduce a reliability-scaling mechanism that modulates epistemic uncertainty while preserving each expert's predictive mean, enabling coherent fusion via MoNIG aggregation.
\item Through experiments on PDBbind, BDB2020+, and both SARS-CoV-2 Mpro dataset and 5HT2A receptor case study, we demonstrate superior uncertainty calibration and up to 25.7\% MAE reduction via selective prediction.
\end{enumerate}

Together, these results position RELIABLE-BA as a practical approach to uncertainty-aware virtual screening in drug discovery. Accurate binding affinity prediction with calibrated uncertainty is particularly critical for challenging therapeutic targets such as GPCRs, which account for over 34\% of FDA-approved drug targets yet remain difficult to model due to their conformational flexibility and limited representation in binding affinity training data~\cite{Hauser_2017_GPCR}. Our 5HT2A receptor case study confirms that RELIABLE-BA maintains strong performance on this clinically relevant target class. To prove its usage in real world, we further validated the method on a SARS-CoV-2 Mpro dataset, a clinically meaningful dataset known to be independent of the training data \cite{Li_2026_LPPDBBind}.

\section{Related Work}
\subsection{Binding Affinity Prediction}

Binding affinity quantifies the strength of interaction between a protein and a ligand, typically expressed as dissociation constants ($K_d$, $K_i$) or their negative logarithms (p$K_d$, p$K_i$) ~\cite{landrum2024combining}. Accurate prediction is essential for computational drug discovery, as experimental validation is costly and time-consuming, allowing only a small fraction of candidate compounds to be tested~\cite{lam2025navigating}.

Existing approaches span a wide spectrum of modeling assumptions. Physics-based approaches include molecular docking with empirical scoring functions, such as AutoDock Vina ~\cite{trott2010autodock} and Glide ~\cite{friesner2004glide},  as well as free energy perturbation (FEP) ~\cite{wang2015accurate}. While docking methods enable efficient large-scale screening, FEP achieves substantially higher accuracy ($\sim$1 kcal/mol error) at the cost of extensive sampling and prohibitive computational expense.

Machine learning approaches offer faster inference with competitive accuracy. Structure-based methods include 3D convolutional networks that operate on voxelized protein--ligand complexes (GNINA~\cite{mcnutt2021gnina}, KDEEP~\cite{jimenez2018kdeep}, Pafnucy~\cite{stepniewska2018development}). Sequence-based models such as BIND~\cite{bind} combine protein language model embeddings with ligand graph neural networks, avoiding dependence on explicit 3D structures. More recently, generative docking methods based on diffusion (DynamicBind~\cite{lu2024dynamicbind}) and flow matching (FlowDock~\cite{flowdock}) jointly predict binding poses and affinities.

Despite this diversity, each approach carries distinct biases. Physics-based docking relies primarily on structural information, making it less dependent on ligand data coverage but limited in capturing complex energetics. Structure-based ML models are sensitive to pose quality, while sequence-based models may miss geometric interactions between protein and ligand~\cite{bind}. Critically, most methods output only point estimates without calibrated confidence scores, leaving practitioners unable to distinguish reliable predictions from unreliable ones.

\subsection{Aggregation of Affinity Prediction Engines}
A common strategy for improving binding affinity prediction is to aggregate outputs from multiple engines. Prior work has explored two main paradigms: consensus scoring and ensemble learning, motivated by the observation that no single engine consistently outperforms others across all target classes~\cite{Wang_2016_comprehensive}.

Consensus scoring aggregates predictions from independent docking or scoring functions using heuristic rules such as rank-based aggregation or score normalization~\cite{Yang_2005_consensus, NhatPhuong_2023_consensus, Scardino_2022_consensus_survey}. ~\citeauthor{Ericksen_2017_MLconsensus} showed that machine learning consensus scoring outperforms naive combinations, but also found that the best-performing engine varies by target, making it difficult to know which engines to trust a priori.

Ensemble learning formulates aggregation within a machine learning framework. ~\citeauthor{Ballester_2010_RFscore} pioneered machine learning scoring functions using Random Forests ~\cite{Ballester_2010_RFscore}, while recent work has explored deep learning ensembles~\cite{Uddin_2024_EBA} and confidence estimation through ensemble disagreement~\cite{Rayka_2024_ENSscore}.

Recent advances in uncertainty quantification offer principled alternatives~\cite{Rayka_2025_UQ_binding}. Evidential methods have demonstrated practical impact: \citeauthor{Soleimany_2021_evidential} \cite{Soleimany_2021_evidential} showed that evidential uncertainty improved virtual screening hit rates, and EviDTI~\cite{EviDTI_2025_NatComm} validated uncertainty-guided predictions with experimental binding assays. UAMRL employed Normal-Inverse-Gamma distributions for uncertainty-aware fusion of multiple modalities within a single model ~\cite{Xu_2025_UAMRL}.

Despite these advances, existing methods either apply uniform weighting across engines or fuse modalities within a single model, without learning context-dependent reliability for each prediction source. In practice, different engines exhibit sample-level biases, which is a gap RELIABLE-BA can address through learned reliability scaling.

\subsection{Uncertainty Quantification and Evidential Learning}

Uncertainty quantification (UQ) characterizes predictive confidence rather than producing point estimates alone. A key distinction is between aleatoric uncertainty, arising from inherent data noise, and epistemic uncertainty, reflecting limited model knowledge~\cite{Kendall_2017_uncertainties}. Common approaches of UQ include deep ensembles~\cite{Lakshminarayanan_2017_ensembles}, Monte Carlo Dropout~\cite{gal2016dropout}, SWA-Gaussian~\cite{maddox2019simple}, and sparse variational Gaussian processes~\cite{hensman2015scalable}.

Another common approach is evidential regression, which is a probabilistic framework in which a neural network predicts the parameters of a Normal--Inverse-Gamma (NIG) prior over the mean and variance of a Gaussian likelihood, yielding closed-form uncertainty decomposition without sampling or ensembles~\cite{amini2020deep}. 

~\citeauthor{ma2021trustworthy} \cite{ma2021trustworthy} introduced the Mixture of NIG (MoNIG) framework for multi-source evidential fusion, combining NIG distributions through a closed-form summation operator that accumulates both within-expert uncertainty and inter-expert disagreement. While MoNIG provides a principled aggregation rule, existing applications rely on static weighting schemes that do not adapt to input-dependent predictor reliability.

\section{RELIABLE-BA: RELIABiLity-aware Evidential fusion for Binding Affinity}
\label{sec:RELIABLE-BA}

\begin{figure*}[t]
  \centering
  \includegraphics[width=\textwidth]{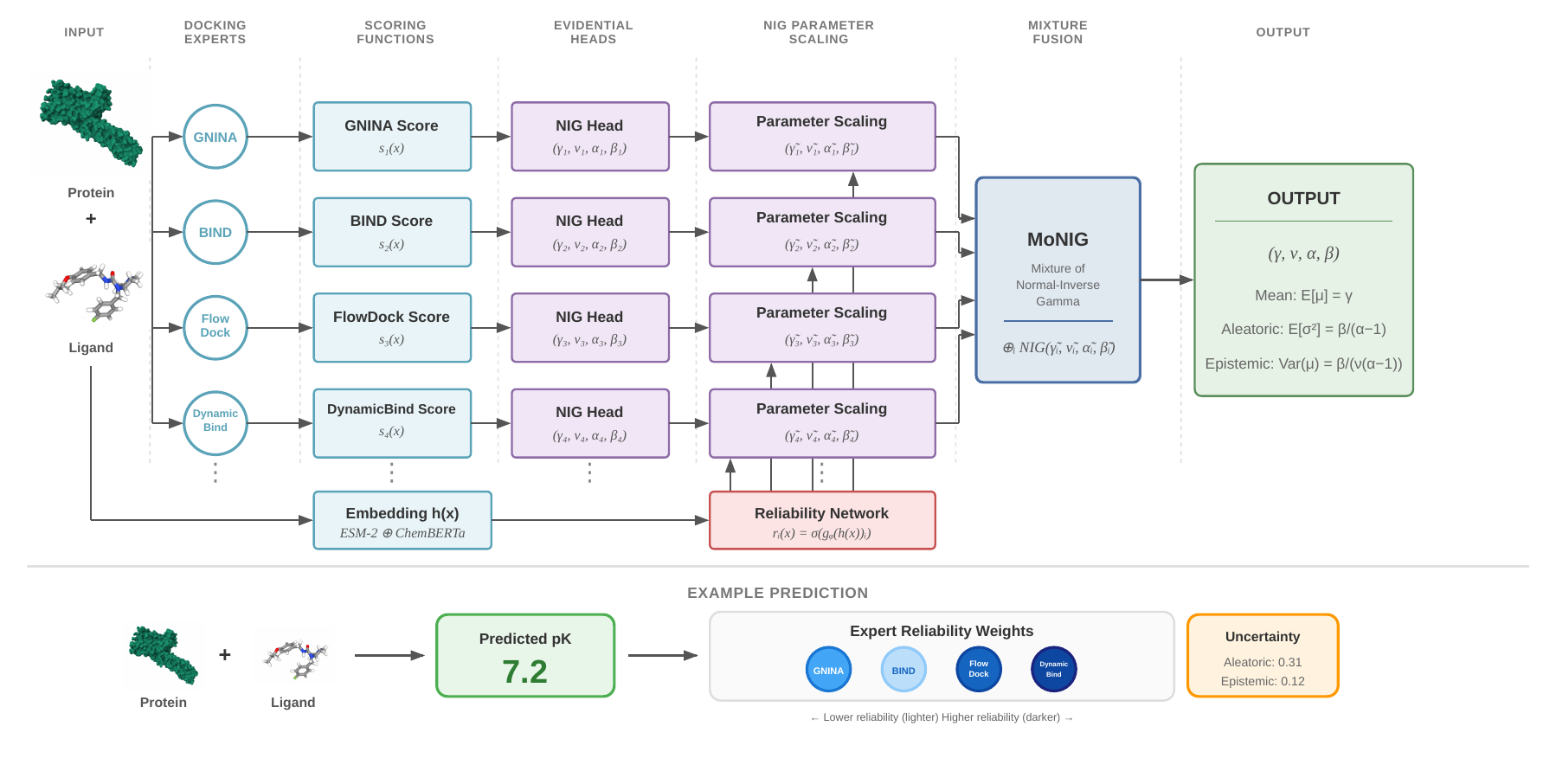}
  \caption{Overview of the RELIABLE-BA architecture. Evidential heads map each engine's score to NIG parameters, which are modulated by context-dependent reliability scores and fused via MoNIG aggregation. An example prediction illustrates per-engine reliability weights (darker = higher reliability) alongside decomposed uncertainty estimates.}
  \Description{Architecture of RELIABLE-BA.}
  \label{fig:RELIABLE-BA_overview}
\end{figure*}

\subsection{Problem Setup}
\label{sec:problem_setup}

We consider the problem of predicting protein--ligand binding affinity by integrating outputs from multiple binding affinity prediction engines.
Let $x$ denote a protein--ligand complex, and $y \in \mathbb{R}$ the ground-truth binding affinity.
Assume $M$ docking engines are available, each producing a scalar score $s_i(x)$ for engine $i$.
In addition, a shared representation $\mathbf{h}(x) \in \mathbb{R}^d$ is extracted by concatenating frozen ESM-2 protein embeddings ~\cite{lin2022language} and ChemBERTa ligand embeddings ~\cite{chithrananda2020chemberta}.
The goal is to output a predictive distribution $p(y \mid x)$ that captures both the predicted value and its associated uncertainty, rather than producing a single point estimate.

\subsection{Evidential Expert Modeling}
\label{sec:evidential_expert}
Each engine is modeled as an evidential expert that outputs a probabilistic prediction over binding affinity.
For engine $i$, we assume a Gaussian likelihood
\begin{equation}
y \mid \mu, \sigma^2 \sim \mathcal{N}(\mu, \sigma^2),
\end{equation}
with a Normal--Inverse-Gamma (NIG) prior ~\cite{amini2020deep} placed over the unknown mean and variance,
\begin{equation}
(\mu, \sigma^2) \sim \mathrm{NIG}(\gamma, \nu, \alpha, \beta).
\end{equation}
Rather than treating these as fixed parameters, we learn the NIG parameters via a neural network.
Specifically, for each engine $i$, an MLP $f_\theta^{(i)}$ maps the engine score $s_i(x)$ 
to the four evidential parameters:
\begin{equation}
(\gamma_i, \nu_i, \alpha_i, \beta_i) = f_\theta^{(i)}\big(s_i(x)\big).
\end{equation}

We enforce $\nu_i > 0$ and $\beta_i > 0$ via softplus activation, and $\alpha_i > 1$ via $\alpha_i = 1 + \mathrm{softplus}(\cdot)$ to ensure the expected variance $\mathbb{E}[\sigma^2]$ remains finite.

The four parameters admit the following interpretation: $\gamma_i$ represents the predicted mean, $\nu_i$ controls confidence in the mean estimate (analogous to pseudo-observation count), while $\alpha_i$ and $\beta_i$ govern the shape and scale of uncertainty over the variance.
From these parameters, the expected prediction and uncertainty components are
\begin{align}
\mathbb{E}[\mu] &= \gamma_i, \\
\mathbb{E}[\sigma^2] &= \frac{\beta_i}{\alpha_i - 1}, \\
\mathrm{Var}(\mu) &= \frac{\beta_i}{\nu_i(\alpha_i - 1)}.
\end{align}
We interpret $\mathbb{E}[\mu]$ as the point prediction, $\mathbb{E}[\sigma^2]$ as aleatoric uncertainty (irreducible data noise), and $\mathrm{Var}(\mu)$ as epistemic uncertainty (model uncertainty due to limited evidence).

\subsection{Reliability-Aware Uncertainty Modeling}
\label{sec:reliability_network}
\paragraph{Network}

Since each docking engine operates on different mechanisms with unique inductive biases, context-dependent reliability will differ across protein--ligand complexes.
To model this behavior, RELIABLE-BA introduces a reliability network that estimates predictive trustworthiness conditioned on the input.
The reliability network takes the shared embedding $\mathbf{h}(x)$ as input and returns a reliability score for each engine,
\begin{equation}
r_i(x) = \sigma\big(g_\phi(\mathbf{h}(x))_i\big), \quad r_i(x) \in (0,1),
\end{equation}
where $g_\phi$ is an MLP and $\sigma$ denotes the sigmoid function.
The output $r_i(x)$ reflects how trustworthy engine $i$ is for the given complex.
The scores quantify predictive trustworthiness conditioned on molecular context, reflecting factors such as protein family, binding pocket geometry, and structural complexity.

\paragraph{ Scaling}
\label{sec:reliability_scaling}

The reliability scores are used to modulate epistemic uncertainty without altering each expert’s predictive bias.
Specifically, RELIABLE-BA applies reliability scaling to the evidential parameters as
\begin{align}
\tilde{\gamma}_i &= \gamma_i, \\
\tilde{\nu}_i &= r_i(x)\, \nu_i, \\
\tilde{\alpha}_i &= r_i(x)\, \alpha_i + \big(1 - r_i(x)\big), \\
\tilde{\beta}_i &= r_i(x)\, \beta_i.
\end{align}
This transformation preserves each expert’s predicted mean while reducing its effective evidence when reliability is low.
Consequently, epistemic uncertainty increases for unreliable experts, while the expected aleatoric variance remains unchanged under the NIG parameterization.
We ensure numerical validity by constraining $r_i(x)$ away from zero, guaranteeing $\tilde{\nu}_i>0$, $\tilde{\beta}_i>0$, and $\tilde{\alpha}_i>1$.

\subsection{MoNIG Aggregation}
\label{sec:monig}

The reliability-scaled experts are fused using a Mixture of Normal--Inverse-Gamma (MoNIG) aggregation.
Let $(\tilde{\gamma}_i, \tilde{\nu}_i, \tilde{\alpha}_i, \tilde{\beta}_i)$ denote the reliability-scaled evidential parameters produced by engine $i$.

Following prior work of MoNIG ~\cite{ma2021trustworthy}, we aggregate multiple Normal--Inverse-Gamma distributions using the NIG summation operator, denoted by $\oplus$, which defines a closed-form fusion rule that preserves conjugacy.
The fusion of $M$ scaled experts is given by
\begin{equation}
\mathrm{NIG}(\gamma,\nu,\alpha,\beta)
=
\bigoplus_{i=1}^M
\mathrm{NIG}(\tilde{\gamma}_i,\tilde{\nu}_i,\tilde{\alpha}_i,\tilde{\beta}_i).
\label{eq:nig_sum}
\end{equation}

The fused mean is computed as an evidence-weighted average,
\begin{equation}
\gamma
=
\frac{\sum_{i=1}^M \tilde{\nu}_i \tilde{\gamma}_i}{\sum_{i=1}^M \tilde{\nu}_i},
\label{eq:fused_mean}
\end{equation}
with total evidence
\begin{equation}
\nu = \sum_{i=1}^M \tilde{\nu}_i .
\label{eq:fused_nu}
\end{equation}

The remaining Normal--Inverse-Gamma parameters are aggregated according to the NIG summation operator.
Specifically, the fused shape parameter is
\begin{equation}
\alpha
=
\sum_{i=1}^M \tilde{\alpha}_i
+
\frac{M-1}{2},
\label{eq:fused_alpha}
\end{equation}
and the fused scale parameter is
\begin{equation}
\beta
=
\sum_{i=1}^M \tilde{\beta}_i
+
\frac{1}{2}
\sum_{i=1}^M
\tilde{\nu}_i
\bigl(\tilde{\gamma}_i - \gamma\bigr)^2 .
\label{eq:fused_beta}
\end{equation}

This aggregation accumulates both within-expert uncertainty and disagreement among expert means, with the second term in Eq.~\eqref{eq:fused_beta} increasing epistemic uncertainty when engines provide conflicting predictions.
While the correctness of the NIG summation operator without reliability scaling has been established in prior work, we provide a formal proof in Appendix \ref{appendix:proof} that the summation remains valid under the proposed reliability scaling.

Given the fused Normal--Inverse-Gamma parameters, predictive uncertainty decomposes into aleatoric and epistemic components as
\begin{equation}
\underbrace{\mathbb{E}[\sigma^2]}_{\text{aleatoric}}
=
\frac{\beta}{\alpha - 1},
\qquad
\underbrace{\mathrm{Var}(\mu)}_{\text{epistemic}}
=
\frac{\beta}{\nu(\alpha - 1)}.
\end{equation}

\subsection{Joint Evidential Learning Objective}
\label{sec:overall_learning}

RELIABLE-BA is trained end-to-end with joint supervision at both the expert and fused levels (Figure~\ref{fig:RELIABLE-BA_overview}).
For each complex $x$ with ground-truth binding affinity $y$, engine-specific branches output evidential parameters
$m_i(x) = (\gamma_i, \nu_i, \alpha_i, \beta_i)$.

We define the overall training objective as the sum of expert-level evidential losses and the fused-level evidential loss:
\begin{equation}
\mathcal{L}(w)
=
\sum_{i=1}^{M}\mathcal{L}_{i}(w)
+
\mathcal{L}_{\mathrm{MoNIG}}(w)
,
\label{eq:overall_loss}
\end{equation}
where each expert loss is given by
\begin{equation}
\mathcal{L}_{i}(w)
=
\mathcal{L}^{\mathrm{NLL}}_{i}(w)
+
\lambda\,\mathcal{L}^{\mathrm{R}}_{i}(w),
\label{eq:expert_loss}
\end{equation}
and the fused loss is
\begin{equation}
\mathcal{L}_{\mathrm{MoNIG}}(w)
=
\mathcal{L}^{\mathrm{NLL}}_{\mathrm{MoNIG}}(w)
+
\lambda\,\mathcal{L}^{\mathrm{R}}_{\mathrm{MoNIG}}(w).
\label{eq:fused_loss}
\end{equation}

Here, $\mathcal{L}^{\mathrm{NLL}}_{i}$ and $\mathcal{L}^{\mathrm{NLL}}_{\mathrm{MoNIG}}$ denote the negative log marginal likelihoods under the expert-level and fused Student-$t$ predictive distributions, respectively.
The regularization terms $\mathcal{L}^{\mathrm{R}}_{i}$ and $\mathcal{L}^{\mathrm{R}}_{\mathrm{MoNIG}}$ follow standard evidential regression and penalize excessive evidence when prediction errors are large.
The coefficient $\lambda$ controls the trade-off between data fit and uncertainty calibration.

\section{Experiments}
We organize our empirical analysis around the following research questions (RQs):

\begin{itemize}
\item \textbf{RQ1}: Does RELIABLE-BA show competitive point prediction accuracy over individual docking engines and aggregation baselines?
\item \textbf{RQ2}: Does RELIABLE-BA produce better-calibrated and more informative uncertainty estimates?
\item \textbf{RQ3}: Does RELIABLE-BA support risk-aware decision making through selective prediction?
\item \textbf{RQ4}: How do reliability modeling and expert disagreement contribute to RELIABLE-BA’s performance?
\end{itemize}
\subsection{Dataset and Experimental Setup}

\subsubsection{Dataset and Evaluation Protocol}
We evaluate RELIABLE-BA on the PDBbind 2020 benchmark ~\cite{wang2005pdbbind} using a time-based split to prevent temporal leakage, which was originally curated by prior works~\cite{lu2024dynamicbind, flowdock, stark2022equibind}.
Within this split, we exclude IC50 affinity measurements, as IC50 values are assay-dependent and not directly comparable to thermodynamic binding constants ($K_i$/$K_d$) without knowledge of assay conditions~\cite{Kalliokoski2013, landrum2024combining}. We have also utilized BDB2020+ \cite{Li_2026_LPPDBBind}, which is an independent, challenging dataset for fair benchmarking. Dataset details are included in Appendix ~\ref{appendix:dataset}.

Protein and ligand inputs are converted into fixed-dimensional representations using ESM-2 ($\texttt{esm2\_t6\_8M\_UR50D}$) \cite{lin2022language} for proteins and ChemBERTa ($\texttt{ChemBERTa-77M-MTR}$) \cite{chithrananda2020chemberta} for ligands. All methods operate on identical 704-dimensional protein--ligand embeddings.

\subsubsection{Evidential Expert Models}
RELIABLE-BA integrates four docking engines as evidential experts, selected to maximize methodological diversity and complementary inductive biases:

\begin{itemize}
\item \textbf{GNINA }(Version 1.3): A Vina-derived docking engine that performs Markov chain Monte Carlo (MCMC) pose sampling followed by CNN-based rescoring on grid-based protein--ligand representations ~\cite{mcnutt2021gnina}.
\item \textbf{BIND }(Version 1.4): A sequence-based affinity predictor that combines a ligand graph neural network with protein language model embeddings (ESM-2) via cross-attention ~\cite{bind}.
\item \textbf{FlowDock }(Version 0.0.3): A geometric generative docking method based on conditional flow matching, learning continuous transformations from apo to holo protein--ligand complexes ~\cite{flowdock}.
\item \textbf{DynamicBind }(Version 1.0): An E(3)-equivariant diffusion-based docking model that learns smooth energy landscapes for ligand-specific binding ~\cite{lu2024dynamicbind}.
\end{itemize}

These experts differ substantially in input modality, architectural assumptions, and algorithms, making them well-suited for uncertainty-aware fusion. Note that the framework is modular and additional non-redundant docking engines can be incorporated as new evidential experts without architectural changes.

\subsubsection{Baselines}
We compare RELIABLE-BA against a comprehensive set of uncertainty-aware baselines, including: (i) Deterministic MLP regression, (ii) Single evidential regression (NIG) \cite{amini2020deep} (iii) heteroscedastic Gaussian regression \cite{Kendall_2017_uncertainties}, (iv) MC Dropout \cite{gal2016dropout}, (v) Deep Ensembles with mean--variance estimation \cite{Lakshminarayanan_2017_ensembles}, (vi) SWA-Gaussian (SWAG) \cite{maddox2019simple}, and (vii) Sparse Variational Gaussian Process (SVGP) \cite{hensman2015scalable}.
We have also compared with traditional aggregation methods including consensus scoring, ensemble scoring, learned weight ensemble (LWE), and softmax mixture of experts (Softmax MoE). Detailed architectural choices and hyperparameter settings for each method are provided in Appendix \ref{appendix:baselines}.

\subsubsection{Evaluation Metrics}
We assess performance along two complementary dimensions:

\textbf{Point prediction accuracy}: Mean Absolute Error (MAE), Root Mean Squared Error (RMSE), Pearson correlation, and coefficient of determination ($R^2$).

\textbf{Uncertainty quality}: Expected Calibration Error (ECE), Continuous Ranked Probability Score (CRPS), and Negative Log-Likelihood (NLL). Uncertainty quality metrics were calculated using Uncertainty Toolbox \cite{chung2021uncertainty}. 

\subsection{Experimental Results}
\begin{table*}[t]
\centering
\caption{Point prediction and uncertainty evaluation on PDBbind time-split benchmark. Results averaged over 20 seeds (± standard deviation). Best results in \textbf{bold}, second best \underline{underlined} among aggregation methods.}
\small
\resizebox{\textwidth}{!}{%
\begin{tabular}{lccccccc}
\toprule
\textbf{Method} & \textbf{MAE} $\downarrow$ & \textbf{RMSE} $\downarrow$ & \textbf{Corr} $\uparrow$ & $\mathbf{R^2}$ $\uparrow$ & \textbf{ECE} $\downarrow$ & \textbf{CRPS} $\downarrow$ & \textbf{NLL} $\downarrow$ \\
\midrule
\multicolumn{8}{l}{\textit{Individual Scoring Engines}} \\
DynamicBind & 0.8095 & 1.0840 & 0.8507 & 0.6806 & - & - & - \\
FlowDock & 1.0155 & 1.2776 & 0.7590 & 0.5563 & - & - & - \\
GNINA & 1.2609 & 1.5621 & 0.6646 & 0.3367 & - & - & - \\
BIND & 1.3898 & 1.7417 & 0.5415 & 0.1754 & - & - & - \\
\midrule
\multicolumn{8}{l}{\textit{Traditional Aggregation Methods}} \\
Consensus Scoring & 0.9356 & 1.1896 & 0.7912 & 0.6153 & - & - & - \\
Ensemble Scoring & 0.9231 & 1.1533 & 0.8148 & 0.6384 & - & - & - \\
LWE & 0.8380 & 1.0750 & 0.8350 & 0.6860 & - & - & - \\
Softmax MoE & 0.8610 ± 0.0200 & 1.0890 ± 0.0180 & 0.8260 ± 0.0050 & 0.6780 ± 0.0110 & - & - & - \\
\midrule
\multicolumn{8}{l}{\textit{Learned Aggregation Methods}} \\
Baseline & 0.8630 ± 0.0287 & 1.0991 ± 0.0360 & 0.8267 ± 0.0107 & 0.6713 ± 0.0214 & - & - & - \\
Gaussian & 0.8543 ± 0.0230 & 1.0961 ± 0.0303 & 0.8309 ± 0.0085 & 0.6732 ± 0.0182 & \underline{0.0376 ± 0.0164} & 0.6153 ± 0.0164 & 1.5272 ± 0.0291 \\
MC Dropout & \underline{0.8435 ± 0.0242} & \underline{1.0796 ± 0.0256} & 0.8290 ± 0.0081 & \underline{0.6830 ± 0.0151} & 0.1009 ± 0.0211 & 0.6286 ± 0.0145 & 1.5893 ± 0.0249 \\
Deep Ensemble & 0.8444 ± 0.0246 & 1.0887 ± 0.0324 & \underline{0.8344 ± 0.0057} & 0.6775 ± 0.0196 & 0.0540 ± 0.0130 & \underline{0.6134 ± 0.0165} & 1.5290 ± 0.0243 \\
SVGP & 0.8680 ± 0.0325 & 1.1162 ± 0.0320 & 0.8184 ± 0.0094 & 0.6610 ± 0.0195 & 0.0810 ± 0.0311 & 0.6416 ± 0.0166 & 1.5971 ± 0.0310 \\
SWAG & 0.8549 ± 0.0241 & 1.0961 ± 0.0281 & 0.8285 ± 0.0090 & 0.6732 ± 0.0169 & 0.0395 ± 0.0147 & 0.6155 ± 0.0151 & \underline{1.5271 ± 0.0248} \\
Evidential Regression & 0.8595 ± 0.0297 & 1.1109 ± 0.0359 & 0.8279 ± 0.0062 & 0.6642 ± 0.0221 & 0.0518 ± 0.0265 & 0.6227 ± 0.0196 & 1.5381 ± 0.0335 \\
\midrule
\multicolumn{8}{l}{\textit{Ours}} \\
\textbf{RELIABLE-BA} & \textbf{0.8247 ± 0.0137} & \textbf{1.0473 ± 0.0203} & \textbf{0.8457 ± 0.0031} & \textbf{0.7017 ± 0.0117} & \textbf{0.0149 ± 0.0055} & \textbf{0.5856 ± 0.0113} & \textbf{1.4664 ± 0.0237} \\
\bottomrule
\end{tabular}%
}
\label{tab:main}
\end{table*}

\begin{table*}[t]
\centering
\caption{Point prediction and uncertainty evaluation on BDB2020+ independent test set. Results averaged over 20 seeds (± standard deviation). Best results in \textbf{bold}, second best \underline{underlined} among aggregation methods.}
\small
\resizebox{\textwidth}{!}{%
\begin{tabular}{lccccccc}
\toprule
\textbf{Method} & \textbf{MAE} $\downarrow$ & \textbf{RMSE} $\downarrow$ & \textbf{Corr} $\uparrow$ & $\mathbf{R^2}$ $\uparrow$ & \textbf{ECE} $\downarrow$ & \textbf{CRPS} $\downarrow$ & \textbf{NLL} $\downarrow$ \\
\midrule
\multicolumn{8}{l}{\textit{Individual Scoring Engines}} \\
DynamicBind & 0.8540 & 1.0860 & 0.4630 & 0.0110 & - & - & - \\
FlowDock & 0.8630 & 1.1480 & 0.5290 & -0.1040 & - & - & - \\
GNINA & 0.7310 & 1.0450 & 0.4920 & 0.0840 & - & - & - \\
BIND & 2.1000 & 2.3360 & 0.3100 & -3.5750 & - & - & - \\
\midrule
\multicolumn{8}{l}{\textit{Traditional Aggregation Methods}} \\
Consensus Scoring & 0.7390 & 0.9970 & 0.5870 & 0.1660 & - & - & - \\
Ensemble Scoring & 0.7990 & 1.0020 & \textbf{0.6200} & 0.1580 & - & - & - \\
LWE & 0.7270 & 0.9620 & 0.5700 & 0.2240 & - & - & - \\
Softmax MoE & 0.7560 ± 0.0400 & 0.9740 ± 0.0470 & \underline{0.6150 ± 0.0240} & 0.2040 ± 0.0760 & - & - & - \\
\midrule
\multicolumn{8}{l}{\textit{Learned Aggregation Methods}} \\
Baseline & 0.7650 ± 0.0610 & 0.9900 ± 0.0770 & 0.5860 ± 0.0240 & 0.1750 ± 0.1340 & - & - & - \\
Gaussian & 0.7400 ± 0.0470 & 0.9510 ± 0.0540 & 0.5920 ± 0.0230 & 0.2390 ± 0.0900 & 0.0820 ± 0.0280 & 0.5350 ± 0.0280 & \underline{1.4150 ± 0.0430} \\
MC Dropout & 0.8100 ± 0.0570 & 1.0450 ± 0.0640 & 0.5590 ± 0.0310 & 0.0820 ± 0.1140 & 0.1430 ± 0.0250 & 0.6200 ± 0.0310 & 1.6160 ± 0.0420 \\
Deep Ensemble & \underline{0.7230 ± 0.0350} & \textbf{0.9310 ± 0.0420} & 0.5950 ± 0.0210 & \textbf{0.2720 ± 0.0660} & 0.1040 ± 0.0240 & \underline{0.5300 ± 0.0190} & 1.4240 ± 0.0290 \\
SVGP & 0.7730 ± 0.0580 & 1.0040 ± 0.0590 & 0.5550 ± 0.0450 & 0.1520 ± 0.0990 & 0.1430 ± 0.0410 & 0.5990 ± 0.0320 & 1.5910 ± 0.0620 \\
SWAG & 0.7480 ± 0.0430 & 0.9670 ± 0.0590 & 0.5860 ± 0.0370 & 0.2140 ± 0.0970 & 0.0830 ± 0.0260 & 0.5420 ± 0.0290 & 1.4300 ± 0.0430 \\
Evidential Regression & 0.7570 ± 0.0610 & 0.9680 ± 0.0740 & 0.5840 ± 0.0250 & 0.2110 ± 0.1310 & \underline{0.0700 ± 0.0360} & 0.5440 ± 0.0420 & 1.4160 ± 0.0710 \\
\midrule
\multicolumn{8}{l}{\textit{Ours}} \\
\textbf{RELIABLE-BA} & \textbf{0.7220 ± 0.0220} & \underline{0.9320 ± 0.0180} & 0.5870 ± 0.0050 & \textbf{0.2720 ± 0.0280} & \textbf{0.0550 ± 0.0160} & \textbf{0.5150 ± 0.0100} & \textbf{1.3750 ± 0.0170} \\
\bottomrule
\end{tabular}%
}
\label{tab:bdb2020}
\end{table*}
\subsubsection{RQ1: Point Prediction Performance}

Table~\ref{tab:main} reports point prediction performance on the PDBbind time-split test set. Among learned and traditional aggregation methods, RELIABLE-BA achieves the strongest overall performance, attaining the lowest MAE (0.8247), lowest RMSE (1.0473), highest correlation (0.8457), and highest $R^2$ (0.7017). RELIABLE-BA is also comparable to the strongest individual engine (DynamicBind: MAE 0.8095, RMSE: 1.0840) while providing calibrated uncertainty estimates that individual engines lack.

Table ~\ref{tab:bdb2020} reports results on BDB2020+, an independent test set with zero overlap with training data. RELIABLE-BA achieves the best MAE (0.722) and $R^2$ (0.272), with competitive RMSE (0.932, second to Deep Ensemble's 0.931) again showing its high point prediction performance.

\subsubsection{RQ2: Uncertainty Calibration and Quality}

Table~\ref{tab:main} summarizes the uncertainty evaluation results on PDBbind time-split test set. RELIABLE-BA consistently outperforms all uncertainty-aware baselines in calibration and distributional quality, achieving the lowest ECE, CRPS, and NLL. Compared to evidential regression, RELIABLE-BA reduces ECE by 71.2\%, CRPS by 6.0\%, and NLL by 4.7\%, indicating substantially improved calibration and sharper predictive distributions. 

Table~\ref{tab:bdb2020} with BDB2020+ also shows that RELIABLE-BA consistently outperforms all uncertainty-aware baselines, again achieving the lowest ECE, CRPS, and NLL. Compared to evidential regression, RELIABLE-BA reduces ECE by 21.4\%, CRPS by 5.3\%, and NLL by 2.9\%, confirming that the calibration improvements transfer to an independent validation set.

\subsubsection{RQ3: Selective Prediction and Risk--Coverage Tradeoff}

\begin{figure}[t]
  \centering
  \includegraphics[width=0.8\columnwidth]{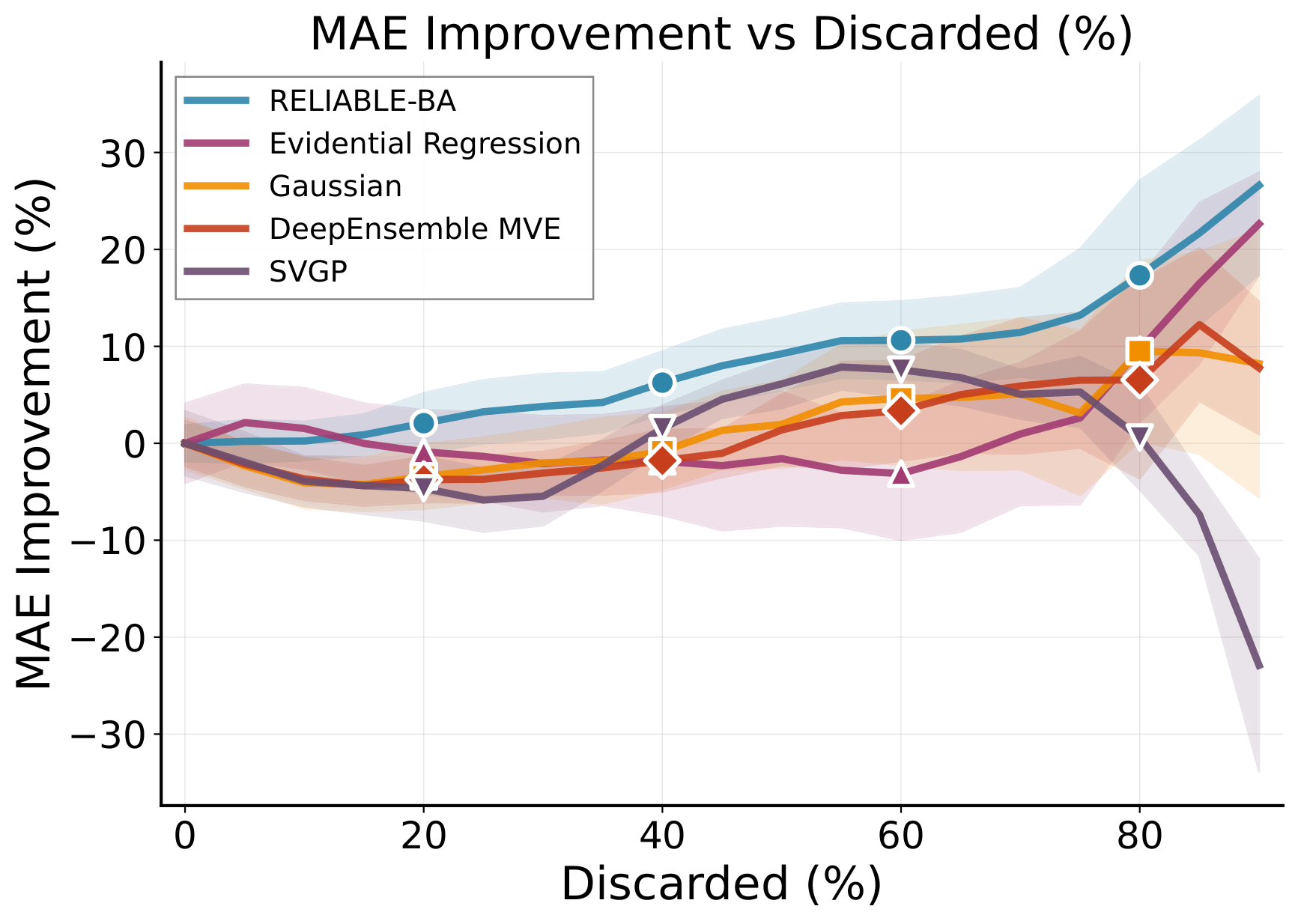}
  \caption{Selective predictive analysis on five models.}
  \Description{Selective predictive analysis on five models.}
  \label{fig:risk_coverage}
\end{figure}

We evaluate whether RELIABLE-BA’s uncertainty estimates meaningfully correlate with prediction difficulty using selective prediction. Figure~\ref{fig:risk_coverage}  reports relative MAE improvement as increasingly uncertain samples are discarded. RELIABLE-BA achieves consistent 
improvement across all discard thresholds, reaching 25.7\%  relative error reduction when retaining only the most confident 10\% of predictions. RELIABLE-BA not only produces well-calibrated uncertainty estimates, but also ranks samples by prediction difficulty more effectively, enabling principled risk-aware filtering in virtual screening pipelines.

\subsubsection{RQ4: Reliability Modeling and Expert Disagreement}

\begin{figure}[t]
  \centering
  \includegraphics[width=\columnwidth]{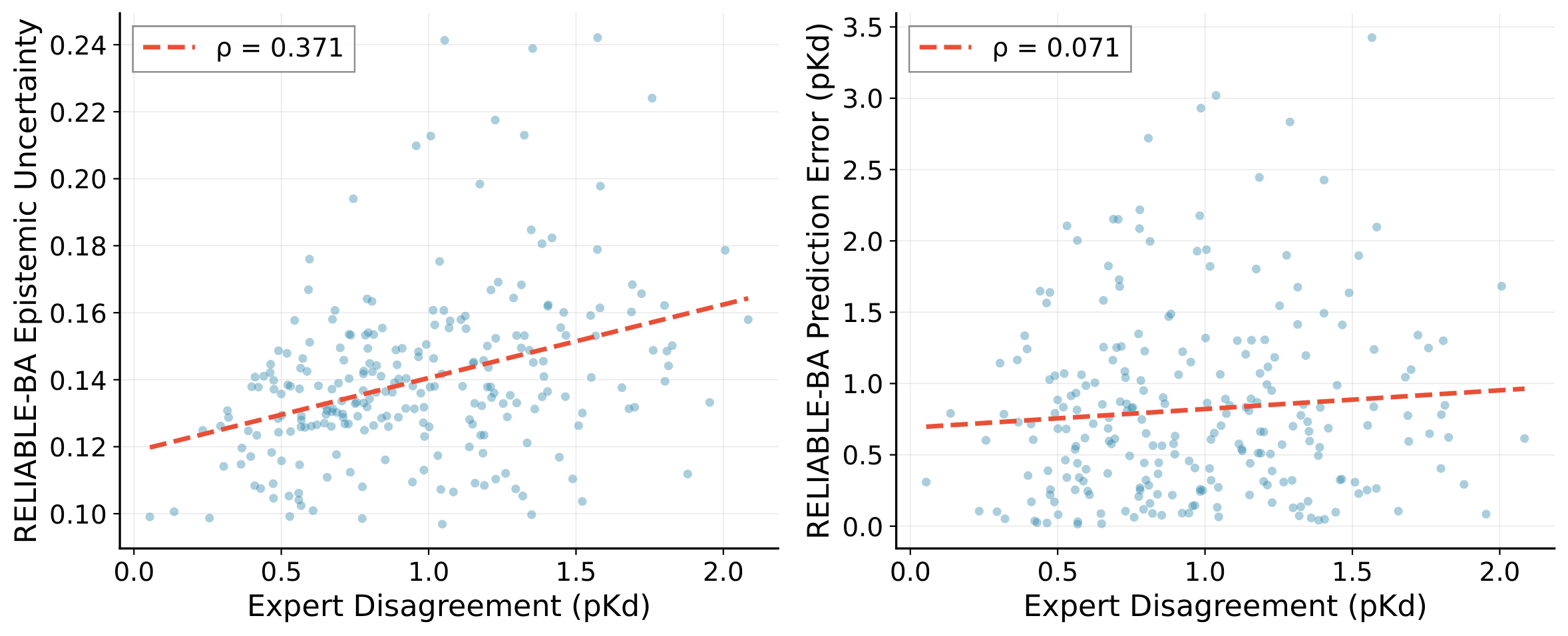}
  \caption{Correlation between expert disagreement and epistemic uncertainty of RELIABLE-BA (left) and expert disagreement and prediction error of RELIABLE-BA (right).}
  \Description{Correlation between expert disagreement and epistemic uncertainty of RELIABLE-BA (left) and expert disagreement and prediction error of RELIABLE-BA (right).}
  \label{fig:disagree}
\end{figure}

We analyze how RELIABLE-BA’s learned reliability interacts with expert disagreement. Figure~\ref{fig:disagree} (left) shows a clear positive correlation between inter-expert disagreement and RELIABLE-BA’s epistemic uncertainty. As expert predictions diverge, RELIABLE-BA appropriately assigns higher epistemic uncertainty, treating disagreement as a signal of epistemic risk rather than noise.

Figure~\ref{fig:disagree} (right) shows that prediction error itself is only weakly correlated with expert disagreement, indicating that RELIABLE-BA does not force consensus at the level of the predictive mean. Instead, disagreement is absorbed through reliability-aware scaling of epistemic uncertainty, preserving stable point predictions while accurately reflecting uncertainty. This design enables RELIABLE-BA to remain relatively robust under heterogeneous expert behavior.

\subsection{Ablation Studies}
We conduct ablation studies to isolate the contribution of key components in RELIABLE-BA, with a particular focus on reliability modeling and expert diversity. Specifically, we evaluate five ablations: (i) removing reliability scaling (No Reliability Scaling), (ii) using only engine scores for reliability estimation without embeddings (Scores Only Reliability), (iii) using uniform reliability weights instead of context-dependent reliability (Uniform Reliability), (iv) aggregating expert predictions using uniform weights without learned reliability-based weighting (Uniform Weight Aggregation), and (v) varying the number of docking engines used in the fusion. For engine diversity experiments, we select engines based on their reliability distribution: 2 engines (DynamicBind, FlowDock) and 3 engines (DynamicBind, FlowDock, BIND). Detailed selection choices are shown in Appendix \ref{appendix:reliability}.

\begin{table}[t]
\caption{Ablation study on RELIABLE-BA components. Results averaged over 20 seeds.}
\centering
\small
\resizebox{\columnwidth}{!}{%
\begin{tabular}{lccccccc}
\toprule
\textbf{Method} & \textbf{MAE} $\downarrow$ & \textbf{RMSE} $\downarrow$ & \textbf{Corr} $\uparrow$ & $\mathbf{R^2}$ $\uparrow$ & \textbf{ECE} $\downarrow$ & \textbf{CRPS} $\downarrow$ & \textbf{NLL} $\downarrow$ \\
\midrule
RELIABLE-BA & \textbf{0.8247} & \textbf{1.0473} & \textbf{0.8457} & \textbf{0.7017} & 0.0149 & \textbf{0.5856} & \textbf{1.4664} \\
No Reliability Scaling & 0.8356 & 1.0613 & 0.8410 & 0.6937 & 0.0147 & 0.5939 & 1.4771 \\
Scores Only Reliability & 0.8385 & 1.0679 & 0.8414 & 0.6899 & \textbf{0.0146} & 0.5972 & 1.4851 \\
Uniform Weight Aggregation & 0.8383 & 1.0555 & 0.8394 & 0.6972 & 0.0209 & 0.5950 & 1.4796 \\
Uniform Reliability & 0.8418 & 1.0701 & 0.8404 & 0.6886 & 0.0279 & 0.5997 & 1.5004 \\
\bottomrule
\end{tabular}%
}
\label{tab:ablation}
\end{table}

As shown in Table ~\ref{tab:ablation}, RELIABLE-BA achieves the best point prediction performance across most ablation variants, with the lowest MAE (0.8247) and RMSE (1.0473) and highest correlation (0.8457) and $R^2$ (0.7017). The primary benefit of reliability modeling is further reflected in uncertainty quality: removing molecular context from reliability estimation or using uniform weights consistently degrades calibration. While ECE is comparable across variants, RELIABLE-BA achieves the best CRPS (0.5856) and NLL (1.4664), demonstrating that context-dependent reliability modeling improves both predictive accuracy and uncertainty calibration.

\begin{table}[t]
\caption{Effect of the number of docking engines in RELIABLE-BA. Results averaged over 20 seeds.}
\centering
\small
\resizebox{\columnwidth}{!}{%
\begin{tabular}{lccccccc}
\toprule
\textbf{Method} & \textbf{MAE} $\downarrow$ & \textbf{RMSE} $\downarrow$ & \textbf{Corr} $\uparrow$ & $\mathbf{R^2}$ $\uparrow$ & \textbf{ECE} $\downarrow$ & \textbf{CRPS} $\downarrow$ & \textbf{NLL} $\downarrow$ \\
\midrule
4 Engines & \textbf{0.8247} & \textbf{1.0473} & \textbf{0.8457} & \textbf{0.7017} & \textbf{0.0149} & \textbf{0.5856} & \textbf{1.4664} \\
3 Engines & 0.8471 & 1.1084 & 0.8448 & 0.6657 & 0.0349 & 0.6124 & 1.5190 \\
2 Engines & 0.8719 & 1.0889 & 0.8340 & 0.6775 & 0.0792 & 0.6245 & 1.5621 \\
\bottomrule
\end{tabular}%
}
\label{tab:ablation_engines}
\end{table}

Table~\ref{tab:ablation_engines} shows that incorporating additional engines improves both accuracy and uncertainty quality. Adding a fourth engine reduces MAE by 5.4\% relative to two engines and by 2.6\% relative to three engines. Calibration benefits are more pronounced: ECE decreases by 81.2\% compared to two engines and by 57.3\% compared to three engines. These results confirm that RELIABLE-BA effectively leverages the complementary strengths of heterogeneous predictors, with each additional engine providing gains in prediction accuracy and calibration.

\subsection{Case Study I: SARS-CoV-2 Mpro Dataset}
We have conducted case study evaluation on the SARS-CoV-2 main protease dataset (Mpro), an independent real-world benchmark disjoint from our training data \cite{Li_2026_LPPDBBind}. Evaluating on the SARS-CoV-2 main protease dataset, a target central to the
COVID-19 pandemic response, demonstrates the real-world impact of this study.

RELIABLE-BA outperforms all individual engines in terms of point prediction with best MAE ($0.532 \pm 0.021$), RMSE ($0.644 \pm 0.024$) and correlation ($0.695 \pm 0.006$) compared to individual engines like GNINA (MAE: 0.631, RMSE: 0.768, Correlation: 0.586) and DynamicBind (MAE: 0.643, RMSE: 	0.780, Correlation: 0.683). RELIABLE-BA also showed strong uncertainty calibration (CRPS = 0.391, NLL = 1.144), demonstrating effective generalization to systematic out-of-domain evaluation.

\subsection{Case Study II: 5HT2A Receptor}

\begin{figure}[t]
\centering
\includegraphics[width=0.8\columnwidth]{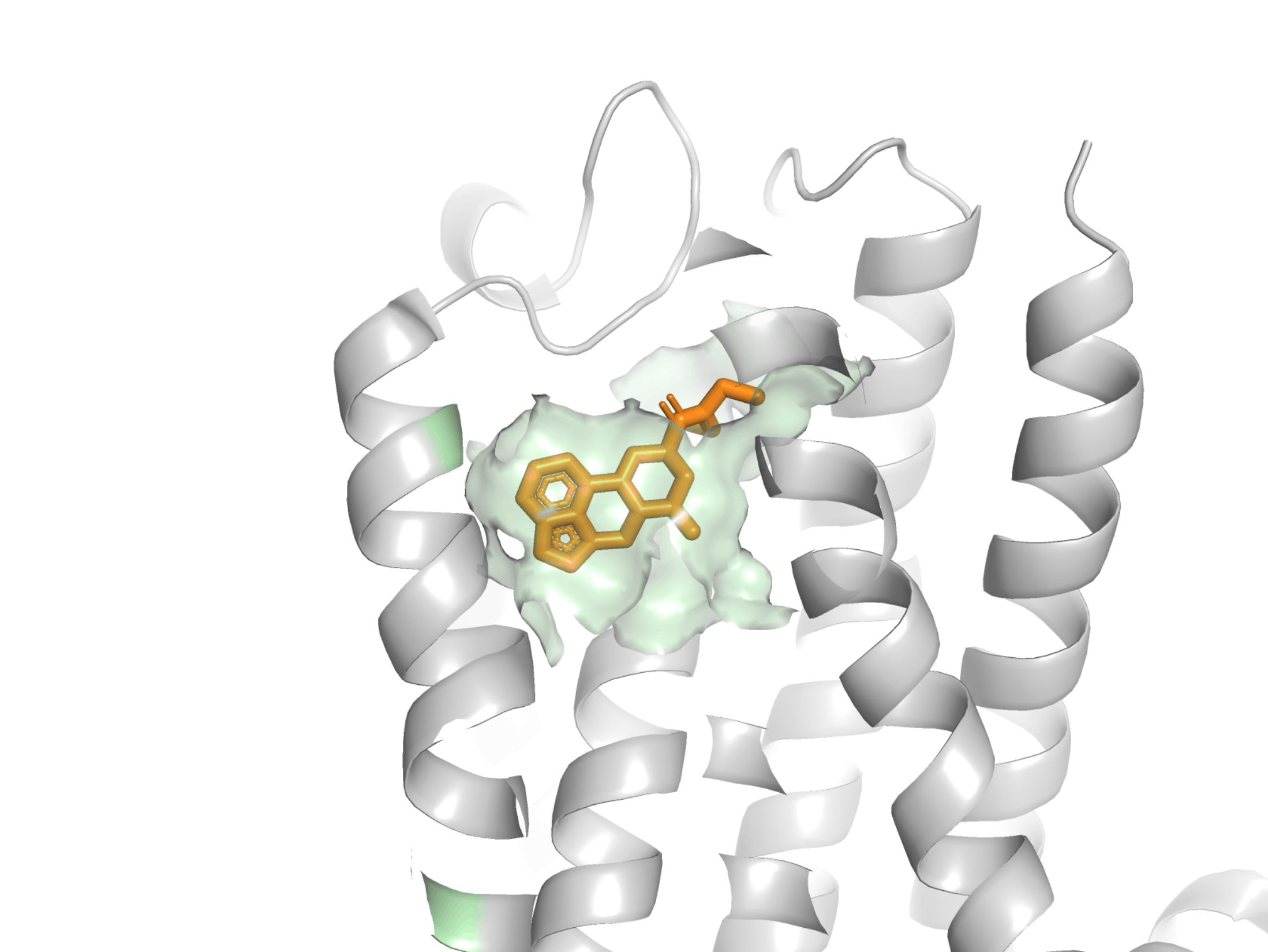}
\caption{Experimentally resolved structure of the human 5HT2A receptor (PDB ID: 7WC6) used in the case study evaluation.}
\Description{Structure of human 5HT2A receptor.}
\label{fig:7wc6}
\end{figure}

The 5HT2A receptor is a class A G protein–coupled receptor (GPCR) that plays a central role in serotonergic signaling and is a key therapeutic target for neurological and psychiatric disorders, including depression and anxiety (Figure ~\ref{fig:7wc6}) \cite{schmidt1995role}. From a computational perspective, 5HT2A provides a realistic evaluation setting for affinity prediction under distribution shift, as the resolved receptor structure is not represented in the training split and the ligand set exhibits substantial functional heterogeneity. This setting is well-suited for examining the behavior of uncertainty-aware fusion methods, and thus can prove that this study can expand to virtual-screening settings.

We select PDB structure 7WC6 as the protein target and evaluate all methods on distinct ligands, averaged over 20 random seeds. RELIABLE-BA achieves an MAE of $1.002 \pm 0.04$ and an RMSE of $1.264 \pm 0.043$, outperforming the strongest individual engine, DynamicBind (MAE: 1.060, RMSE: 1.360), by 5.5\% and 7.1\%, respectively. Notably, uncertainty calibration remains strong under distribution shift. RELIABLE-BA achieves ECE of 0.0233, comparable to in-distribution performance (0.0149), with CRPS of 0.7143 and NLL of 1.6622, showing competitive uncertainty quantification. 

\section{Conclusion}
In this paper, we proposed RELIABLE-BA, a reliability-aware evidential framework that fuses heterogeneous docking engines through context-dependent uncertainty modulation. By modeling each engine as an evidential expert and scaling epistemic uncertainty according to learned reliability, RELIABLE-BA produces calibrated predictive distributions that enable effective selective prediction, reducing MAE by 25.7\% when retaining the most confident 10\% of predictions. Experiments on PDBbind and BDB2020+ demonstrate strong point prediction accuracy and well-calibrated uncertainty estimates, with additional validation on the SARS-CoV-2 Mpro dataset and the 5HT2A GPCR confirming applicability to practical drug targets. Moreover, RELIABLE-BA's modular design allows it to integrate heterogeneous docking and scoring engines, making it broadly applicable to large-scale virtual screening pipelines. Future work will explore extending RELIABLE-BA to larger docking ensembles and incorporating pose-level uncertainty to further enhance applicability in drug discovery workflows.

\section{Limitations and Ethical Considerations}
This work does not involve human participants or personally identifiable information. All experiments are conducted on publicly available benchmark datasets for protein–ligand binding affinity prediction (PDBbind 2020 \cite{wang2005pdbbind}, ChEMBL \cite{gaulton2012chembl}, BDB2020+ \cite{Li_2026_LPPDBBind}), derived from previously published experimental measurements and used in accordance with their original licenses.

As a limitation, RELIABLE-BA relies on the availability of multiple binding affinity prediction engines and may be less effective when only a few engines are available or when their predictions are highly correlated. The learned reliability scores reflect context-dependent predictive trustworthiness rather than physical or mechanistic correctness and should not be interpreted as causal guarantees. Our evaluation relies on established benchmark datasets, which may contain inherent biases or data leakage despite mitigation efforts. Finally, RELIABLE-BA introduces computational overhead from multi-engine inference, though this overhead is minimal compared to the cost of individual docking engines (Appendix~\ref{appendix:compute}).

\begin{acks}
This work was supported in part by the U.S. National Science Foundation under Grant No.~2125142 to Y.L. and by the National Institutes of Health under Grant No.~R35GM153437 to V.K. GPU resources were provided through an NVIDIA Academic Grant.
\end{acks}

\section*{GenAI Disclosure}
Generative AI tools were used in a limited capacity. Claude was utilized for minor code construction, while ChatGPT was used for grammar and writing refinement.

\bibliographystyle{ACM-Reference-Format}
\balance
\bibliography{reference}

@article{panahandeh2025,
  title={A comprehensive review of neural network-based approaches for drug--target interaction prediction},
  author={Panahandeh, Fatemeh and Mansouri, Najme},
  journal={Molecular Diversity},
  pages={1--48},
  year={2025},
  publisher={Springer}
}

@inproceedings{
cao2026pinfdit,
title={{PINFD}iT: Energy-Based Physics-Informed {Diffusion} {Transformers} for General-purpose Time Series Tasks},
author={Defu Cao and Wen Ye and Yizhou Zhang and Sam Griesemer and Yan Liu},
booktitle={The Fourteenth International Conference on Learning Representations},
year={2026},
url={https://openreview.net/forum?id=EphTlUJ4XN}
}

@inproceedings{cao2023estimating,
  title={Estimating treatment effects from irregular time series observations with hidden confounders},
  author={Cao, Defu and Enouen, James and Wang, Yujing and Song, Xiangchen and Meng, Chuizheng and Niu, Hao and Liu, Yan},
  booktitle={Proceedings of the AAAI Conference on Artificial Intelligence},
  volume={37},
  pages={6897--6905},
  year={2023}
}

@article{lam2025navigating,
  title={Navigating structure-based drug discovery with emerging innovations in physics-and knowledge-based approaches},
  author={Lam, Jordy Homing and Katritch, Vsevolod},
  journal={npj Drug Discovery},
  volume={2},
  number={1},
  pages={29},
  year={2025},
  publisher={Nature Publishing Group UK London}
}

@article{passaro2025boltz2,
  title={Boltz-2: Towards Accurate and Efficient Binding Affinity Prediction},
  author={Passaro, Saro and Corso, Gabriele and Wohlwend, Jeremy and Reveiz, Mateo and Thaler, Stephan and Somnath, Vignesh Ram and Getz, Noah and Portnoi, Tally and Roy, Julien and Stark, Hannes and others},
  journal={bioRxiv},
  year={2025},
  doi={10.1101/2025.06.14.659707}
}

@article{Wang_2016_comprehensive,
  author    = {Wang, Zhe and Sun, Huiyong and Yao, Xiaojun and Li, Dan and Xu, Lei and Li, Youyong and Tian, Sheng and Hou, Tingjun},
  title     = {Comprehensive evaluation of ten docking programs on a diverse set of protein--ligand complexes: the prediction accuracy of sampling power and scoring power},
  journal   = {Physical Chemistry Chemical Physics},
  volume    = {18},
  number    = {18},
  pages     = {12964--12975},
  year      = {2016},
  doi       = {10.1039/C6CP01555G}
}

@article{buttenschoen2024posebusters,
  title={PoseBusters: AI-based docking methods fail to generate physically valid poses or generalise to novel sequences},
  author={Buttenschoen, Martin and Morris, Garrett M and Deane, Charlotte M},
  journal={Chemical Science},
  volume={15},
  number={9},
  pages={3130--3139},
  year={2024},
  publisher={Royal Society of Chemistry}
}

@article{Scardino_2022_consensus_survey,
  author={Blanes-Mira, Clara and Fern{\'a}ndez-Aguado, Pilar and de Andr{\'e}s-L{\'o}pez, Jorge and Fern{\'a}ndez-Carvajal, Asia and Ferrer-Montiel, Antonio and Fern{\'a}ndez-Ballester, Gregorio},
  title     = {Comprehensive survey of consensus docking for high-throughput virtual screening},
  journal   = {Molecules},
  volume    = {28},
  number    = {1},
  pages     = {175},
  year      = {2022},
  doi       = {10.3390/molecules28010175}
}

@article{Ericksen_2017_MLconsensus,
  author    = {Ericksen, Spencer S and Wu, Haozhen and Zhang, Huikun and Michael, Lauren A and Newton, Michael A and Hoffmann, F Michael and Wildman, Scott A},
  title     = {Machine Learning Consensus Scoring Improves Performance Across Targets in Structure-Based Virtual Screening},
  journal   = {Journal of Chemical Information and Modeling},
  volume    = {57},
  number    = {7},
  pages     = {1579--1590},
  year      = {2017},
  doi       = {10.1021/acs.jcim.7b00153}
}

@article{Yang_2005_consensus,
  author    = {Yang, Jinn-Moon and Chen, Yen-Fu and Shen, Tsai-Wei and Kristal, Bruce S. and Hsu, D. Frank},
  title     = {Consensus Scoring Criteria for Improving Enrichment in Virtual Screening},
  journal   = {Journal of Chemical Information and Modeling},
  volume    = {45},
  number    = {4},
  pages     = {1134--1146},
  year      = {2005},
  doi       = {10.1021/ci050034w}
}

@article{NhatPhuong_2023_consensus,
  author    = {Nhat Phuong, Do and Flower, Darren R. and Chattopadhyay, Subhagata and Chattopadhyay, Amit K.},
  title     = {Towards effective consensus scoring in Structure-Based virtual screening},
  journal   = {Interdisciplinary Sciences: Computational Life Sciences},
  volume    = {15},
  number    = {1},
  pages     = {131--145},
  year      = {2023},
  doi       = {10.1007/s12539-022-00546-8}
}

@article{Uddin_2024_EBA,
  title={Ensembling methods for protein-ligand binding affinity prediction},
  author={Mohamed Abdul Cader, Jiffriya and Newton, MA Hakim and Rahman, Julia and Mohamed Abdul Cader, Akmal Jahan and Sattar, Abdul},
  journal={Scientific Reports},
  volume={14},
  number={1},
  pages={24447},
  year={2024},
  publisher={Nature Publishing Group UK London}
}

@article{Rayka_2024_ENSscore,
  title     = {An ensemble-based approach to estimate confidence of predicted protein--ligand binding affinity values},
  author={Rayka, Milad and Mirzaei, Morteza and Mohammad Latifi, Ali},
  journal={Molecular Informatics},
  volume={43},
  number={4},
  pages={e202300292},
  year={2024},
  publisher={Wiley Online Library},
  doi       = {10.1002/minf.202300292}
}

@article{Ballester_2010_RFscore,
  author    = {Ballester, Pedro J. and Mitchell, John B. O.},
  title     = {A machine learning approach to predicting protein--ligand binding affinity with applications to molecular docking},
  journal   = {Bioinformatics},
  volume    = {26},
  number    = {9},
  pages     = {1169--1175},
  year      = {2010},
  doi       = {10.1093/bioinformatics/btq112}
}

@inproceedings{Lakshminarayanan_2017_ensembles,
  author    = {Lakshminarayanan, Balaji and Pritzel, Alexander and Blundell, Charles},
  title     = {Simple and scalable predictive uncertainty estimation using deep ensembles},
  journal={Advances in neural information processing systems},
  volume={30},
  year={2017}
}

@inproceedings{Kendall_2017_uncertainties,
  author    = {Kendall, Alex and Gal, Yarin},
  title={What uncertainties do we need in bayesian deep learning for computer vision?},
  booktitle = {Advances in Neural Information Processing Systems},
  volume    = {30},
  year      = {2017}
}

@article{Rayka_2025_UQ_binding,
  title={Uncertainty quantification enables reliable deep learning for protein--ligand binding affinity prediction},
  author={Rayka, Milad and Naghavi, S Shahab},
  journal={Scientific Reports},
  volume={15},
  number={1},
  pages={43156},
  year={2025},
  publisher={Nature Publishing Group UK London},
  doi       = {10.1038/s41598-025-27167-7}
}

@article{Soleimany_2021_evidential,
  title={Evidential Deep Learning for Guided Molecular Property Prediction and Discovery},
  author={Soleimany, Ava P and Amini, Alexander and Goldman, Samuel and Rus, Daniela and Bhatia, Sangeeta N and Coley, Connor W},
  journal={ACS central science},
  volume={7},
  number={8},
  pages={1356--1367},
  year={2021},
  publisher={ACS Publications},
  doi       = {10.1021/acscentsci.1c00546}
}

@article{EviDTI_2025_NatComm,
  title={Evidential deep learning-based drug-target interaction prediction},
  author={Zhao, Yanpeng and Xing, Yuting and Zhang, Yixin and Wang, Yifei and Wan, Mengxuan and Yi, Duoyun and Wu, Chengkun and Li, Shangze and Xu, Huiyan and Zhang, Hongyang and others},
  journal={Nature communications},
  volume={16},
  number={1},
  pages={6915},
  year={2025},
  publisher={Nature Publishing Group UK London}
}

@article{Xu_2025_UAMRL,
  title={UAMRL: multi-granularity uncertainty-aware multimodal representation learning for drug-target affinity prediction},
  author={Xu, Wenzhe and Liu, Xiaorong and Wang, Jie and Zhang, Fan and Hu, Dongfeng and Zong, Liansong},
  journal={Bioinformatics},
  volume={41},
  number={10},
  pages={btaf512},
  year={2025},
  publisher={Oxford University Press},
  doi       = {10.1093/bioinformatics/btaf512}
}

@article{amini2020deep,
  title={Deep evidential regression},
  author={Amini, Alexander and Schwarting, Wilko and Soleimany, Ava and Rus, Daniela},
  journal={Advances in neural information processing systems},
  volume={33},
  pages={14927--14937},
  year={2020}
}

@article{ma2021trustworthy,
  title={Trustworthy multimodal regression with mixture of normal-inverse gamma distributions},
  author={Ma, Huan and Han, Zongbo and Zhang, Changqing and Fu, Huazhu and Zhou, Joey Tianyi and Hu, Qinghua},
  journal={Advances in Neural Information Processing Systems},
  volume={34},
  pages={6881--6893},
  year={2021}
}

@article{wang2015accurate,
  title={Accurate and reliable prediction of relative ligand binding potency in prospective drug discovery by way of a modern free-energy calculation protocol and force field},
  author={Wang, Lingle and Wu, Yujie and Deng, Yuqing and Kim, Byungchan and Pierce, Levi and Krilov, Goran and Lupyan, Dmitry and Robinson, Shaughnessy and Dahlgren, Markus K and Greenwood, Jeremy and others},
  journal={Journal of the American Chemical Society},
  volume={137},
  number={7},
  pages={2695--2703},
  year={2015},
  publisher={ACS Publications}
}

@article{trott2010autodock,
  title={AutoDock Vina: improving the speed and accuracy of docking with a new scoring function, efficient optimization, and multithreading},
  author={Trott, Oleg and Olson, Arthur J},
  journal={Journal of Computational Chemistry},
  volume={31},
  number={2},
  pages={455--461},
  year={2010},
  publisher={Wiley}
}

@article{friesner2004glide,
  title={Glide: a new approach for rapid, accurate docking and scoring. 1. Method and assessment of docking accuracy},
  author={Friesner, Richard A and Banks, Jay L and Murphy, Robert B and Halgren, Thomas A and Klicic, Jasna J and Mainz, Daniel T and Repasky, Matthew P and Knoll, Eric H and Shelley, Mee and Perry, Jason K and others},
  journal={Journal of Medicinal Chemistry},
  volume={47},
  number={7},
  pages={1739--1749},
  year={2004},
  publisher={ACS Publications}
}

@article{mcnutt2021gnina,
  title={GNINA 1.0: molecular docking with deep learning},
  author={McNutt, Andrew T and Francoeur, Paul and Aggarwal, Rishal and Masuda, Tomohide and Meli, Rocco and Ragoza, Matthew and Sunseri, Jocelyn and Koes, David Ryan},
  journal={Journal of cheminformatics},
  volume={13},
  number={1},
  pages={43},
  year={2021},
  publisher={Springer}
}

@article{jimenez2018kdeep,
  title={K deep: protein--ligand absolute binding affinity prediction via 3d-convolutional neural networks},
  author={Jim{\'e}nez, Jos{\'e} and Skalic, Miha and Martinez-Rosell, Gerard and De Fabritiis, Gianni},
  journal={Journal of chemical information and modeling},
  volume={58},
  number={2},
  pages={287--296},
  year={2018},
  publisher={ACS Publications}
}

@article{stepniewska2018development,
  title={Development and evaluation of a deep learning model for protein--ligand binding affinity prediction},
  author={Stepniewska-Dziubinska, Marta M and Zielenkiewicz, Piotr and Siedlecki, Pawel},
  journal={Bioinformatics},
  volume={34},
  number={21},
  pages={3666--3674},
  year={2018},
  publisher={Oxford University Press}
}

@article{lu2024dynamicbind,
  title={DynamicBind: predicting ligand-specific protein-ligand complex structure with a deep equivariant generative model},
  author={Lu, Wei and Wu, Qifeng and Zhang, Jixian and Rao, Jiahua and Li, Chengtao and Zheng, Shuangjia},
  journal={Nature Communications},
  volume={15},
  number={1},
  pages={1071},
  year={2024},
  publisher={Nature Publishing Group}
}

@article{flowdock,
    author = {Morehead, Alex and Cheng, Jianlin},
    title = {{FlowDock}: Geometric flow matching for generative protein--ligand docking and affinity prediction},
    journal = {Bioinformatics},
    volume = {41},
    number = {Supplement\_1},
    pages = {{i198--i206}},
    year = {2025},
    publisher={Oxford University Press},
    doi = {10.1093/bioinformatics/btaf187}
}

@article{bind,
  title={Protein language models are performant in structure-free virtual screening},
  author={Lam, Hilbert Yuen In and Guan, Jia Sheng and Ong, Xing Er and Pincket, Robbe and Mu, Yuguang},
  journal={Briefings in Bioinformatics},
  volume={25},
  number={6},
  pages={bbae480},
  year={2024},
  publisher={Oxford University Press}
}

@inproceedings{gal2016dropout,
  author    = {Gal, Yarin and Ghahramani, Zoubin},
  title={Dropout as a bayesian approximation: Representing model uncertainty in deep learning},
  booktitle={international conference on machine learning},
  pages={1050--1059},
  year={2016},
  organization={PMLR}
}

@inproceedings{maddox2019simple,
  title={A simple baseline for bayesian uncertainty in deep learning},
  author={Maddox, Wesley J and Izmailov, Pavel and Garipov, Timur and Vetrov, Dmitry P and Wilson, Andrew Gordon},
  journal={Advances in neural information processing systems},
  volume={32},
  year={2019}
}

@inproceedings{hensman2015scalable,
  title={Scalable variational Gaussian process classification},
  author={Hensman, James and Matthews, Alexander and Ghahramani, Zoubin},
  booktitle={Artificial intelligence and statistics},
  pages={351--360},
  year={2015},
  organization={PMLR}
}

@inproceedings{stark2022equibind,
  title={Equibind: Geometric deep learning for drug binding structure prediction},
  author={St{\"a}rk, Hannes and Ganea, Octavian and Pattanaik, Lagnajit and Barzilay, Regina and Jaakkola, Tommi},
  booktitle={International conference on machine learning},
  pages={20503--20521},
  year={2022},
  organization={PMLR}
}

@article{Li_2026_LPPDBBind,
  title={Leak Proof PDBBind: a reorganized data set of protein--ligand complexes for more generalizable binding affinity prediction},
  author={Li, Jie and Guan, Xingyi and Zhang, Oufan and Sun, Kunyang and Wang, Yingze and Bagni, Dorian and Head-Gordon, Teresa},
  journal={The Journal of Physical Chemistry B},
  volume={130},
  number={2},
  pages={730--740},
  year={2026},
  publisher={ACS Publications}
}

@article{wang2005pdbbind,
  title={The PDBbind database: methodologies and updates},
  author={Wang, Renxiao and Fang, Xueliang and Lu, Yipin and Yang, Chao-Yie and Wang, Shaomeng},
  journal={Journal of medicinal chemistry},
  volume={48},
  number={12},
  pages={4111--4119},
  year={2005},
  publisher={ACS Publications}
}

@article{Graff_2021_active_learning,
  title={Accelerating high-throughput virtual screening through molecular pool-based active learning},
  author={Graff, David E and Shakhnovich, Eugene I and Coley, Connor W},
  journal={Chemical science},
  volume={12},
  number={22},
  pages={7866--7881},
  year={2021},
  publisher={Royal Society of Chemistry}
}

@article{Hauser_2017_GPCR,
  title={Trends in GPCR drug discovery: new agents, targets and indications},
  author={Hauser, Alexander S and Attwood, Misty M and Rask-Andersen, Mathias and Schi{\"o}th, Helgi B and Gloriam, David E},
  journal={Nature reviews Drug discovery},
  volume={16},
  number={12},
  pages={829--842},
  year={2017},
  publisher={Nature Publishing Group UK London}
}

@article{Kalliokoski2013,
  title={Comparability of Mixed IC50 Data--A Statistical Analysis},
  author={Kalliokoski, Tuomo and Kramer, Christian and Vulpetti, Anna and Gedeck, Peter},
  journal={PLoS ONE},
  volume={8},
  number={4},
  pages={e61007},
  year={2013},
  publisher={Public Library of Science San Francisco, USA},
  doi={10.1371/journal.pone.0061007}
}

@article{lin2022language,
  title={Language models of protein sequences at the scale of evolution enable accurate structure prediction},
  author={Lin, Zeming and Akin, Halil and Rao, Roshan and Hie, Brian and Zhu, Zhongkai and Lu, Wenting and dos Santos Costa, Allan and Fazel-Zarandi, Maryam and Sercu, Tom and Candido, Sal and others},
  journal={BioRxiv},
  volume={2022},
  pages={500902},
  year={2022}
}

@article{chithrananda2020chemberta,
  title={ChemBERTa: large-scale self-supervised pretraining for molecular property prediction},
  author={Chithrananda, Seyone and Grand, Gabriel and Ramsundar, Bharath},
  journal={arXiv preprint arXiv:2010.09885},
  year={2020}
}

@article{schmidt1995role,
  title={The role of 5-HT2A receptors in antipsychotic activity},
  author={Schmidt, Christopher J and Sorensen, Stephen M and Kenne, John H and Carr, Albert A and Palfreyman, Michael G},
  journal={Life sciences},
  volume={56},
  number={25},
  pages={2209--2222},
  year={1995},
  publisher={Elsevier}
}

@article{landrum2024combining,
  title={Combining IC50 or K i values from different sources is a source of significant noise},
  author={Landrum, Gregory A and Riniker, Sereina},
  journal={Journal of chemical information and modeling},
  volume={64},
  number={5},
  pages={1560--1567},
  year={2024},
  publisher={ACS Publications}
}

@article{gaulton2012chembl,
  title={ChEMBL: a large-scale bioactivity database for drug discovery},
  author={Gaulton, Anna and Bellis, Louisa J and Bento, A Patricia and Chambers, Jon and Davies, Mark and Hersey, Anne and Light, Yvonne and McGlinchey, Shaun and Michalovich, David and Al-Lazikani, Bissan and others},
  journal={Nucleic acids research},
  volume={40},
  number={D1},
  pages={D1100--D1107},
  year={2012},
  publisher={Oxford University Press}
}

@article{chung2021uncertainty,
  title={Uncertainty Toolbox: an Open-Source Library for Assessing, Visualizing, and Improving Uncertainty Quantification},
  author={Chung, Youngseog and Char, Ian and Guo, Han and Schneider, Jeff and Neiswanger, Willie},
  journal={arXiv preprint arXiv:2109.10254},
  year={2021}
}

@article{breiman1996stacked,
  title={Stacked regressions},
  author={Breiman, Leo},
  journal={Machine Learning},
  volume={24},
  number={1},
  pages={49--64},
  year={1996},
  publisher={Springer},
  doi={10.1007/BF00117832}
}

\appendix
\section{Proof of Validity of Reliability-Scaled MoNIG Aggregation}
\label{appendix:proof}

We prove that RELIABLE-BA's reliability scaling preserves NIG validity and that MoNIG fusion remains valid under this scaling.

\subsection{Preliminaries}

A Normal--Inverse-Gamma distribution $\mathrm{NIG}(\gamma, \nu, \alpha, \beta)$ is \emph{valid} if:
\begin{itemize}
    \item[(V1)] $\nu > 0$
    \item[(V2)] $\alpha > 1$
    \item[(V3)] $\beta > 0$
\end{itemize}

Under these conditions, the moments are well-defined:
\begin{align}
    \mathbb{E}[\mu] &= \gamma, \\
    \mathbb{E}[\sigma^2] &= \frac{\beta}{\alpha - 1} \quad \text{(aleatoric uncertainty)}, \\
    \mathrm{Var}(\mu) &= \frac{\beta}{\nu(\alpha - 1)} \quad \text{(epistemic uncertainty)}.
\end{align}

\subsection{Reliability Scaling Preserves Validity}

Given valid NIG parameters $(\gamma_j, \nu_j, \alpha_j, \beta_j)$ and reliability score $r_j \in (0, 1]$, RELIABLE-BA applies the following transformation:
\begin{align}
    \tilde{\gamma}_j &= \gamma_j, \\
    \tilde{\nu}_j &= r_j \nu_j, \\
    \tilde{\alpha}_j &= 1 + r_j(\alpha_j - 1), \\
    \tilde{\beta}_j &= r_j \beta_j.
\end{align}

\begin{proposition}
If $(\gamma_j, \nu_j, \alpha_j, \beta_j)$ is a valid NIG and $r_j \in (0, 1]$, then $(\tilde{\gamma}_j, \tilde{\nu}_j, \tilde{\alpha}_j, \tilde{\beta}_j)$ is also valid.
\end{proposition}

\begin{proof}
We verify each validity condition:
\begin{itemize}
    \item[(V1)] $\tilde{\nu}_j = r_j \nu_j > 0$ since $r_j > 0$ and $\nu_j > 0$.
    \item[(V2)] $\tilde{\alpha}_j = 1 + r_j(\alpha_j - 1) > 1$ since $r_j > 0$ and $\alpha_j - 1 > 0$.
    \item[(V3)] $\tilde{\beta}_j = r_j \beta_j > 0$ since $r_j > 0$ and $\beta_j > 0$.
\end{itemize}
\end{proof}

\subsection{Effect on Uncertainty Components}

\begin{proposition}[Aleatoric Invariance]
Reliability scaling preserves aleatoric uncertainty: $\tilde{\sigma}^2_{\mathrm{a},j} = \sigma^2_{\mathrm{a},j}$.
\end{proposition}

\begin{proof}
\begin{align}
    \tilde{\sigma}^2_{\mathrm{a},j} 
    = \frac{\tilde{\beta}_j}{\tilde{\alpha}_j - 1} 
    = \frac{r_j \beta_j}{r_j(\alpha_j - 1)} 
    = \frac{\beta_j}{\alpha_j - 1} 
    = \sigma^2_{\mathrm{a},j}.
\end{align}
\end{proof}

\begin{proposition}[Epistemic Scaling]
Reliability scaling increases epistemic uncertainty for unreliable experts: $\tilde{\sigma}^2_{\mathrm{e},j} = \sigma^2_{\mathrm{e},j} / r_j$.
\end{proposition}

\begin{proof}
\begin{align}
    \tilde{\sigma}^2_{\mathrm{e},j} 
    = \frac{\tilde{\beta}_j}{\tilde{\nu}_j(\tilde{\alpha}_j - 1)} 
    = \frac{r_j \beta_j}{r_j \nu_j \cdot r_j(\alpha_j - 1)} 
    = \frac{1}{r_j} \cdot \frac{\beta_j}{\nu_j(\alpha_j - 1)} 
    = \frac{\sigma^2_{\mathrm{e},j}}{r_j}.
\end{align}
\end{proof}

When $r_j < 1$, epistemic uncertainty increases by factor $1/r_j > 1$, reflecting reduced confidence in unreliable experts.

\subsection{MoNIG Fusion Under Reliability Scaling}

Following ~\citeauthor{ma2021trustworthy} \cite{ma2021trustworthy}, the NIG summation operator $\oplus$ fuses $M$ distributions with parameters:
\begin{align}
    \gamma_{\mathrm{f}} &= \frac{\sum_{j=1}^{M} \nu_j \gamma_j}{\sum_{j=1}^{M} \nu_j}, \\
    \nu_{\mathrm{f}} &= \sum_{j=1}^{M} \nu_j, \\
    \alpha_{\mathrm{f}} &= \sum_{j=1}^{M} \alpha_j + \frac{M-1}{2}, \\
    \beta_{\mathrm{f}} &= \sum_{j=1}^{M} \beta_j + \frac{1}{2}\sum_{j=1}^{M} \nu_j (\gamma_j - \gamma_{\mathrm{f}})^2.
\end{align}

\begin{theorem}[Validity of Reliability-Scaled Fusion]
Let $\{(\gamma_j, \nu_j, \alpha_j, \beta_j)\}_{j=1}^{M}$ be $M$ valid NIG distributions with reliability scores $r_j \in (0, 1]$. Then the fused distribution
\begin{align}
    \bigoplus_{j=1}^{M} \mathrm{NIG}(\tilde{\gamma}_j, \tilde{\nu}_j, \tilde{\alpha}_j, \tilde{\beta}_j)
\end{align}
is a valid NIG.
\end{theorem}

\begin{proof}
We verify each validity condition for the fused parameters.

\textbf{Condition (V1):} 
\begin{align}
    \nu_{\mathrm{f}} = \sum_{j=1}^{M} \tilde{\nu}_j = \sum_{j=1}^{M} r_j \nu_j > 0,
\end{align}
since each term $r_j \nu_j > 0$.

\textbf{Condition (V2):} Substituting the scaled $\tilde{\alpha}_j$:
\begin{align}
    \alpha_{\mathrm{f}} 
    &= \sum_{j=1}^{M} \tilde{\alpha}_j + \frac{M-1}{2} \\
    &= \sum_{j=1}^{M} \left(1 + r_j(\alpha_j - 1)\right) + \frac{M-1}{2} \\
    &= M + \sum_{j=1}^{M} r_j(\alpha_j - 1) + \frac{M-1}{2} \\
    &= \frac{3M - 1}{2} + \sum_{j=1}^{M} r_j(\alpha_j - 1).
\end{align}
Since $r_j > 0$ and $\alpha_j > 1$ for all $j$, we have $\alpha_{\mathrm{f}} > \frac{3M-1}{2} \geq 1$ for $M \geq 1$.

\textbf{Condition (V3):} The fused scale parameter is:
\begin{align}
    \beta_{\mathrm{f}} = \underbrace{\sum_{j=1}^{M} r_j \beta_j}_{> 0} + \underbrace{\frac{1}{2}\sum_{j=1}^{M} \tilde{\nu}_j (\tilde{\gamma}_j - \gamma_{\mathrm{f}})^2}_{\geq 0} > 0.
\end{align}
\end{proof}

\begin{corollary}[Reliability-Weighted Fusion Mean]
Under reliability scaling, the fused mean becomes:
\begin{align}
    \gamma_{\mathrm{f}} = \frac{\sum_{j=1}^{M} r_j \nu_j \gamma_j}{\sum_{j=1}^{M} r_j \nu_j}.
\end{align}
Each expert's contribution is weighted by $r_j \nu_j$, the product of reliability and evidence.
\end{corollary}

\subsection{Summary}

The reliability scaling mechanism in RELIABLE-BA satisfies the following properties:
\begin{enumerate}
    \item Scaled parameters remain valid NIG distributions (Proposition 1).
    \item Aleatoric uncertainty is preserved exactly (Proposition 2).
    \item Epistemic uncertainty scales inversely with reliability (Proposition 3).
    \item MoNIG fusion of reliability-scaled experts yields a valid NIG (Theorem 1).
    \item The fused mean weights experts by reliability $\times$ evidence (Corollary 1).
\end{enumerate}

\section{Dataset Curation}
\label{appendix:dataset}
For PDBbind, after filtering invalid entries and complexes with incomplete predictions from engines, the final data consists of 8,742 unique protein–ligand complexes (7,924 train / 576 validation / 242 test). For BDB2020+, we retained 96 complexes after excluding entries with proteins exceeding embedding length limits or incomplete structural annotations. For the 5HT2A case study, we evaluate on 3,868 ligands with valid affinity annotations and docking engine outputs curated from ChEMBL dataset. For SARS-CoV-2 Mpro dataset, we utilized the curated dataset in LP-PDBBind \cite{Li_2026_LPPDBBind}.

\section{Baseline Architectures and Hyperparameters}
\label{appendix:baselines}

All neural network baselines use the same encoder: a 4-layer network (708 → 512 → 256 → 128 → 64) that takes 4 engine scores and 704-dimensional molecular embeddings as input. We use ReLU activations, 20\% dropout, Adam optimizer with learning rate 0.0005, batch size 64, and stop training if validation MAE doesn't improve for 30 epochs.

\paragraph{RELIABLE-BA} Each docking engine has its own 3-layer head  (1 → 256 → 128 → 64) that maps the engine score to four NIG parameters through separate linear projections. A separate reliability network (704 → 128 → 4) takes molecular embeddings and outputs a reliability score for each engine via sigmoid activation. Both components train jointly with evidential loss ($\lambda=0.001$).

\paragraph{Evidential Regression} Single evidential network that predicts four NIG parameters ($\gamma, \nu, \alpha, \beta$) from all inputs combined. Uses shared encoder with 4 linear output heads. Same evidential loss ($\lambda=0.001$) but without per-engine modeling or reliability weighting.

\paragraph{Gaussian} Predicts mean $\mu$ and variance $\sigma^2$ directly. Uses shared encoder with 2 linear output heads. Trained with Gaussian negative log-likelihood. 

\paragraph{Baseline} Standard regression network without uncertainty estimation. Uses shared encoder with single linear output. Trained with mean squared error loss.

\paragraph{Deep Ensemble MVE} Five independently trained Gaussian networks with different random initializations. Final prediction averages across members. 

\paragraph{MC Dropout} Gaussian network with dropout kept active during prediction. Runs 50 forward passes and estimates uncertainty from variation across passes.

\paragraph{SVGP} Gaussian process with learned deep features. Encoder (without dropout) maps to 64 dimensions, followed by RBF kernel GP with 128 inducing points.

\paragraph{SWAG} Gaussian base model that collects weight statistics after epoch 75. At prediction, samples 30 weight configurations from the approximate posterior.

\paragraph{Consensus Scoring} Computes median across engines, then weights each engine by $\exp(-|\text{deviation from median}|)$.

\paragraph{Ensemble Scoring} Arithmetic mean of engines with standard deviation across engines.

\paragraph{LWE} Linear stacking \citep{breiman1996stacked} with four fixed softmax-normalized scalar weights on raw scores, plus one learnable noise scalar for uncertainty.

\paragraph{Softmax MoE} Mixture-of-experts on the full input. Four expert networks (708 → 128 → 64 → 1) combined by an input-dependent softmax gating network, with a separate network for uncertainty.

\section{Hyperparameter Sensitivity}
\label{appendix:hyperparams}

Table~\ref{tab:lambda} shows RELIABLE-BA performance across different values of the evidential regularization coefficient $\lambda$. Lower $\lambda$ yields substantially better calibration (ECE, NLL), while moderate values ($\lambda{=}0.01$) slightly improve point prediction at the expense of calibration. We select $\lambda=0.001$ as it achieves the best calibration with only a marginal increase in MAE, reflecting our emphasis on well-calibrated uncertainty estimates.

\begin{table}[h]
\centering
\small
\caption{Effect of evidential regularization coefficient $\lambda$ on RELIABLE-BA performance. Results averaged over 20 seeds (std in parentheses).}
\label{tab:lambda}
\resizebox{\columnwidth}{!}{%
\begin{tabular}{ccccc}
\toprule
$\lambda$ & MAE $\downarrow$ & RMSE $\downarrow$ & ECE $\downarrow$ & NLL $\downarrow$ \\
\midrule
0.001 & 0.824 (0.015) &  1.045 (0.021) & \textbf{0.014} (0.003) & \textbf{1.463} (0.023) \\
0.005 & 0.816 (0.019) & 1.039 (0.024) & 0.027 (0.011) & 1.474 (0.030) \\
0.01 & \textbf{0.814} (0.017) & \textbf{1.037} (0.022) & 0.042 (0.011) & 1.491 (0.032) \\
0.05 & 0.823 (0.041) & 1.046 (0.051) & 0.078 (0.019) & 1.596 (0.086) \\
\bottomrule
\end{tabular}%
}
\end{table}

\section{Reliability Score Visualization}
\label{appendix:reliability}

\begin{figure}[h]
\centering
\includegraphics[width=0.8\columnwidth]{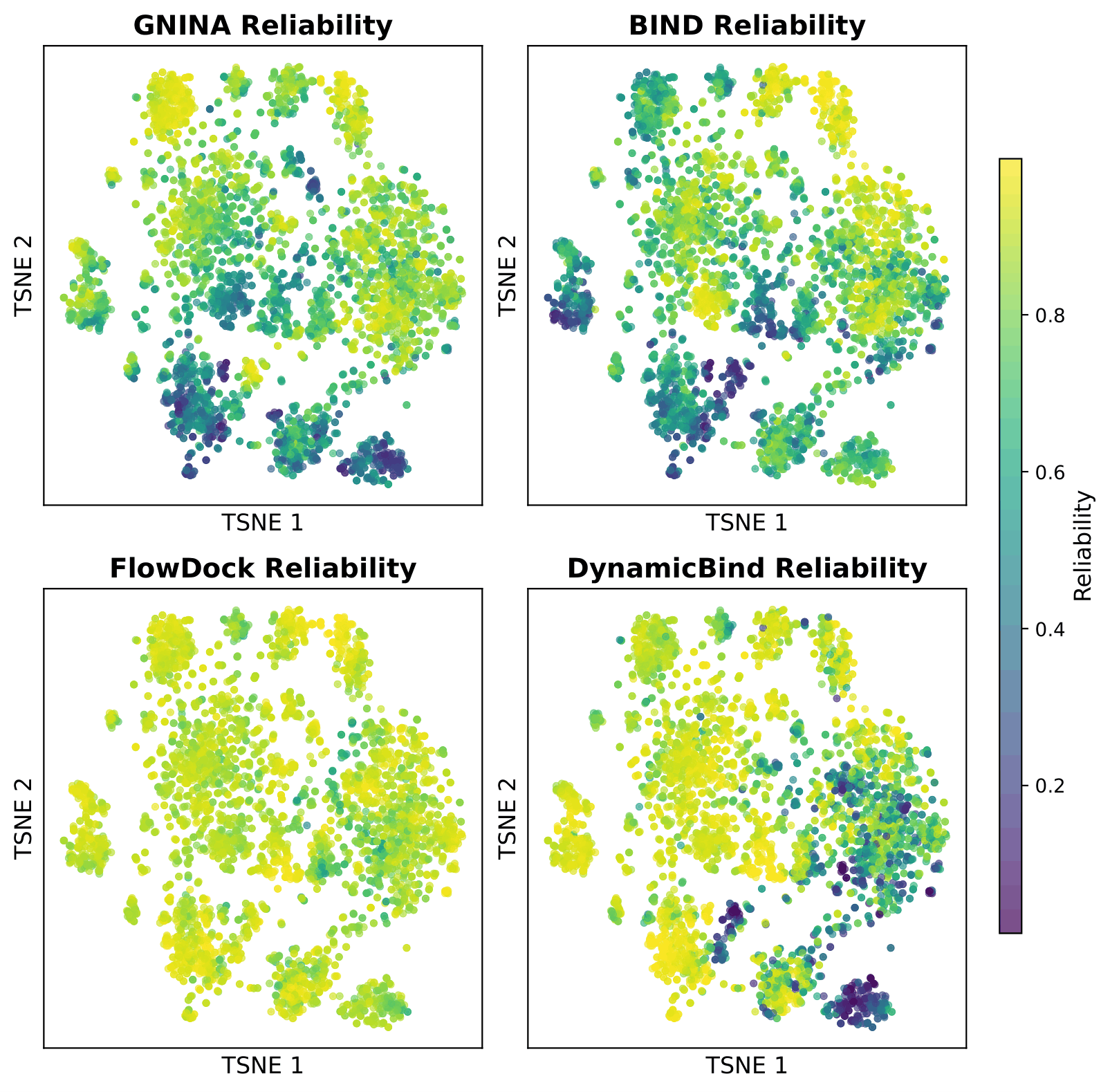}
\caption{t-SNE visualization of protein-ligand embeddings colored by learned reliability scores.}
\Description{t-SNE of protein-ligand embeddings by reliability scores.}
\label{fig:reliability_tsne}
\end{figure}

Figure~\ref{fig:reliability_tsne} visualizes learned reliability scores across protein–ligand embedding space using t-SNE on the PDBBind test set. Each panel shows the reliability assigned to one docking engine, with color indicating reliability (yellow = high, purple = low). FlowDock achieves the highest mean reliability (0.895), followed by DynamicBind (0.81), GNINA (0.719), and BIND (0.706). 

\section{Computational Cost Analysis}
\label{appendix:compute}

Table~\ref{tab:compute} reports training and inference times for all methods on the PDBBind benchmark (7,924 training / 242 test samples). All experiments were conducted on a single NVIDIA RTX 5090 GPU.

\begin{table}[H]
\centering
\small
\caption{Computational cost comparison. Training and inference times in seconds, averaged over 20 runs.}
\label{tab:compute}
\begin{tabular}{lcc}
\toprule
\textbf{Method} & \textbf{Training (s)} & \textbf{Inference (s)} \\
\midrule
Baseline & 9.02 $\pm$ 0.64 & 2.04 $\pm$ 0.00 \\
Gaussian & 10.30 $\pm$ 0.54 & 2.72 $\pm$ 0.02 \\
SWAG & 10.19 $\pm$ 0.28 & 2.73 $\pm$ 0.01 \\
MC Dropout & 13.31 $\pm$ 0.94 & 2.77 $\pm$ 0.01 \\
Evidential Regression & 15.38 $\pm$ 1.06 & 2.74 $\pm$ 0.02 \\
SVGP & 27.85 $\pm$ 2.59 & 2.80 $\pm$ 0.01 \\
Deep Ensemble MVE & 30.14 $\pm$ 1.24 & 2.76 $\pm$ 0.01 \\
\midrule
RELIABLE-BA & 48.38 $\pm$ 4.70 & 2.77 $\pm$ 0.01 \\
\bottomrule
\end{tabular}
\end{table}

RELIABLE-BA has higher total training time (48s) due to per-engine modeling, but training completes in under one minute and inference (11ms per complex) matches all other methods. In practice, upstream docking engines dominate computational cost (seconds to minutes per complex), making RELIABLE-BA's overhead negligible.

\end{document}